\documentclass[twoside]{article}

\usepackage{aistats-frontmatter/fancyhdr}
\usepackage[accepted]{aistats-frontmatter/aistats2025}
\usepackage[page,header]{appendix}
\usepackage{tocloft}
\usepackage{fontawesome5}

\usepackage[numbers,sort&compress]{natbib}
\renewcommand{\bibname}{References}
\renewcommand{\bibsection}{\subsubsection*{\bibname}}
\PassOptionsToPackage{numbers, compress}{natbib}
\usepackage{aistats-frontmatter/preamble}
\usepackage{aistats-frontmatter/macros}

\crefname{equation}{}{}
\crefname{figure}{Figure}{}

\begin{document}
\twocolumn[
  \runningtitle{Signal Recovery from Random Dot-Product Graphs Under Local Differential Privacy}
  \aistatstitle{
    Signal Recovery from Random Dot-Product Graphs \\Under Local Differential Privacy
  }
  \aistatsauthor{ 
    Siddharth Vishwanath {\href{mailto:svishwanath@ucsd.edu}{\faEnvelope[regular]}}
    \And 
    Jonathan Hehir {\href{mailto:research@jonhehir.com}{\faEnvelope[regular]}}
  }
  \runningauthor{Siddharth Vishwanath \& Jonathan Hehir}
  \aistatsaddress{ 
    University of California San Diego 
    \And  
    Penn State University 
  } 
]

\begingroup
\setlength{\abovedisplayskip}{0.5em}
\setlength{\belowdisplayskip}{0.5em}
\setlength{\abovedisplayshortskip}{0.25em}
\setlength{\belowdisplayshortskip}{0.25em}


\begin{abstract}
\vspace*{-0.5em}
    We consider the problem of recovering latent information from graphs under $\varepsilon$-edge local differential privacy where the presence of relationships/edges between two users/vertices remains confidential, even from the data curator. For the class of generalized random dot-product graphs, we show that a standard local differential privacy mechanism induces a specific geometric distortion in the latent positions. Leveraging this insight, we show that consistent recovery of the latent positions is achievable by appropriately adjusting the statistical inference procedure for the privatized graph. Furthermore, we prove that our procedure is nearly minimax-optimal under local edge differential privacy constraints. Lastly, we show that this framework allows for consistent recovery of geometric and topological information underlying the latent positions, as encoded in their persistence diagrams. Our results extend previous work from the private community detection literature to a substantially richer class of models and inferential tasks.
\end{abstract}


\vspace*{-0.5em}
\section{Introduction}\label{introduction}

\vspace*{-0.5em}
Graphs are essential for modeling many systems and are used in a variety of applications, e.g., relationships between people in a social network, associations between individuals based on email records, financial transactions, or website browsing activity. In many cases, the existence, lack thereof, or the nature of such relationships may be sensitive. Without a formal privacy framework, malicious actors can exploit this sensitive information \citep{backstrom2007wherefore,narayanan2008,narayanan2011link}. 

Differential privacy (DP) \citep{dwork2006calibrating} has emerged as the standard for ensuring formal privacy, allowing for population-level inference, while limiting the risk of exposing the contribution of any single individual in a database. This disclosure risk is controlled by the  \textit{privacy budget}, $\e > 0$, where smaller values of $\e$ provide stronger privacy guarantees. For graphs, DP comes in several forms, each with subtle differences between them \citep{li2023private}. These differences are, perhaps, best illuminated by answers to the following questions:
(i)~\textit{What sensitive information needs protection?},
(ii)~\textit{How stringent should the privacy guarantees be?}, and 
(iii)~\textit{Who, if anyone, can be trusted with the data?}

The first question leads to two categories: \textit{node DP} and \textit{edge DP}. In node DP, the identities of the vertices and all their incident relationships are considered sensitive; whereas, in edge DP, the identities of the vertices are publicly available, and the sensitive information is in the presence/absence of relationships between vertices. The second question distinguishes \textit{pure DP} from \textit{approximate DP}. Pure DP, or $\e$-DP, strictly controls disclosure risk using $\e$; whereas, approximate DP, or $(\e, \delta)$-DP, allows for a small failure probability $\delta$. The third question leads to \textit{central DP} vs. \textit{local DP}. In central DP, a \textit{trusted curator} holds the full database and ensures DP in the data release. In local DP \citep{warner1965randomized,kasiviswanathan2011can,duchi2013local}, each individual obfuscates their own data before sharing, removing the need for a trusted curator.

Deciding between node vs. edge DP is constrained by the nature of what is publicly available and what constitutes sensitive information, e.g., in social networks where individuals' identities are public. The choice between pure vs. approximate DP and central vs. local DP depends on considerations of privacy, utility, and practicality. We focus on \textit{$\e$-edge local differential privacy} (\eeldpe{}), where the vertex identities are public, but the strongest form of privacy for the relationships is required.

Commonly referred to as the \textit{privacy--utility tradeoff}, if $\e$ is too large there is risk of leaking sensitive information, while if $\e$ is too small the private information precludes meaningful inference. Addressing this requires (i) a judicious choice of $\e$, (ii) optimal algorithms for sanitizing sensitive information, and (iii) appropriate adjustments in the resulting statistical inference to account for privacy. It is well known that there is no free lunch in DP, and that ``\textit{any meaningful privacy--utility guarantees must come with reasonable assumptions on the data generating mechanism}'' \citep[Section~3]{kifer2011no}. In this work, we consider graphs generated from a flexible, nonparametric framework.

{\noindent \bfseries Generalized random dot-product graphs.}
Random dot-product graphs (RDPGs; \citep{athreya2017statistical}) and generalized random dot-product graphs (GRDPGs; \citep{rubin2022statistical}) provide a flexible framework for analyzing graphs. They cover a wide range of models, including the stochastic block model (SBM; \citep{holland1983stochastic}) and its extensions, such as the degree-corrected SBM (DC-SBM; \citep{karrer2011stochastic}) and mixed-membership SBM (MM-SBM; \citep{airoldi2008mixed}). GRDPGs belong to the family of \textit{latent position models} \citep{hoff2002latent}, as illustrated in Figure~\ref{fig:zoo}. For example, the GRDPG representation of a SBM consists of a mixture of Dirac masses associated with the block memberships. Similarly, MM-SBMs and DC-SBMs admit GRDPGs which lie, respectively, on a simplex and on projective space. Recent work has shown that the adjacency embeddings of GRDPGs recovers meaningful information from its latent positions~\citep{lyzinski2014perfect,rubin2022statistical,solanki2019persistent,rubin2020manifold}.

\begin{figure}[t]
    \centering
    \vspace*{-1em}
    \includegraphics[width=0.55\columnwidth]{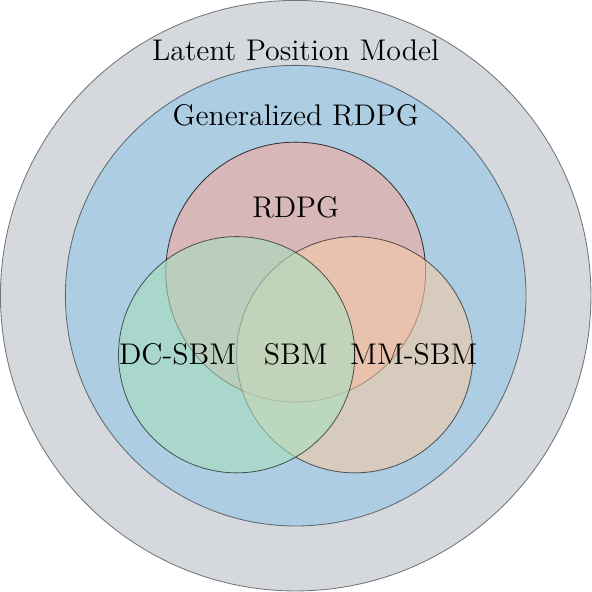}
    \vspace*{-0.5em}
    \captionof{figure}{\label{fig:zoo}Hierarchy of network models.}
    \vspace*{-1em}
\end{figure}

{\noindent \bfseries Contributions.}
Given a graph $\Av$, we consider a simple $\e$-edge DP mechanism called \eflip{} that outputs a differentially private synthetic graph, $\M_\e(\Av)$. \eflip{} has independently appeared in prior work under various aliases \citep[e.g.,][]{mulle2015privacy,karwa2017sharing,qin2017generating,imola2020locally,hehir2021consistent}. We study the impact of \eflip{} on the quality of inference from GRDPGs. Our main contributions are as follows:
\begin{itemize}[labelindent=!, itemindent=*, leftmargin=0em]
    \item We show that the class of GRDPGs are closed under \eflip{}, i.e., if $\Av$ is a GRDPG then $\M_\e(\Av)$ is also a GRDPG (Theorem~\ref{prop:closure}).
    \item Using this insight, if $\Av$ is a GRDPG with latent positions $\X \in \Rnd$, we derive minimax lower bounds for uniformly estimating the latent positions under \eeldpe{} in the $\ell_{2,\infty}$-norm. Our result highlights the interplay between the privacy budget $\e$, the sample size $n$, and the graph sparsity  $\rho_n$ (Theorem~\ref{thm:minimax}).
    \item For sufficiently dense graphs, by adjusting the estimation procedure to account for \eflip{} we show that a privacy-adjusted spectral embedding (\cref{alg:pase}) provides consistent estimates of the latent positions, with near minimax-optimal accuracy up to an $O(\sqrt{\log{n}})$ factor (Theorem~\ref{thm:convergence-rate}). This same gap exists in the best-known bounds for the non-private setting \citep{rubin2022statistical,yan2023minimax}.
    \item Finally, we show how existing results in community detection under $\e$-edge DP can be extended to the broader problem of recovering topological information underlying the latent positions. To this end, we use tools from topological data analysis (TDA) and establish convergence rates for persistence diagram estimation under \eedpy{}.
\end{itemize}

{\noindent \bfseries Related Work.} 
There is an extensive body of work on differential privacy (DP) in graphs, but theoretical work relating to the privacy-utility tradeoff tends to focus on releasing graph statistics, such as degree sequences, triangle counts, and subgraph counts under edge DP \citep[e.g.,][]{hay2009accurate,karwa2011private,karwa2016inference,imola2020locally,jiang2021applications}.

The closest related work to our setting comes from the emerging literature on private community detection \citep{hehir2021consistent,seif2022differentially,chen2023private,he2024a}. In \citep{hehir2021consistent}, community detection for \eflip{} under \eeldpe{} is studied for SBMs and DC-SBMs without sparsity patterns, with utility measured via misclassification rates in approximate clustering. In a similar setting, \cite{chakraborty2024prime} consider estimating the membership probabilities for DC-MM-SBMs under \eeldpe{}. In the central \eecdpe{} setting, \citep{seif2022differentially} consider community detection for SBMs (where the average degree grows as $\log{n}$), focusing on \textit{exact recovery} using semi-definite programming. In \cite{chen2023private}, the analysis is extended to SBMs whose expected degree is flexible, providing bounds for, both, \textit{exact} and \textit{weak recovery}. This line of analysis has also been extended to more general models which relax the SBM assumptions \citep{nguyendifferentially,he2024a}.

Although these works apply similar spectral methods and privacy mechanisms, there are several key distinctions. First, their settings are more focused and restrictive. SBMs and DC-SBMs are specific cases of the more general GRDPGs that we study. Second, they assess utility based on cluster recovery accuracy, whereas we evaluate utility by the $\ell_{2,\infty}$ error in recovering latent positions. The broader framework allows for more complex structures beyond simple assortative clusters, such as hierarchical or topological patterns. Finally, all of our results are in the local DP setting, where there is no centralized trusted curator releasing queries. The price to be paid for this generality is that our upper bounds apply only to a denser regime where the average degree grows at rate $\omega(\log^4{n})$, as opposed to $\omega(\log{n})$ which the corresponds to the information-theoretic threshold. The former is the regime in which the best known recovery bounds for GRDPGs operate.


\vspace*{-0.75em}
\section{Background}
\label{sec:background}

\textbf{Notation.} For $\uv \in \R^d$, $\norm{\uv}$ is the Euclidean norm, and for $\Bv \in \R^{n\times m}$, $\opnorm{\Bv}$ and $\ttinf{\Bv}$ denote the $\ell_2$-- and $\ell_{2,\infty}$--operator norm. ${\X = [X_1 | \cdots | X_n]\tr \in \Rnd}$ denotes the $n\times d$ matrix whose \textit{row vectors} are observations $\Xn = \qty{X_1, \cdots, X_n} \subset \R^d$. For fixed \textit{known} integers $d,p,q$ with $p+q=d$, $\I[p,q]\! =\! \textup{diag}(\onev[p],\!\!-\onev[q])$ is the indefinite identity matrix, and
$$
\o(p,q) = \qty\big{\Q \in \Rdd: \Q\tr\I[p,q]\Q = \Q\I[p,q]\Q\tr = \I[p,q]}
$$ 
is the indefinite orthogonal group with signature $(p,q)$. When $q=0$, $\o(d)$ is the usual orthogonal group. $\bbU(n, d)\!\! =\!\!\{\Uv \in \Rnd: \Uv\tr\Uv = \I[d]\}$ are $n\times d$ matrices with orthonormal columns, and $\bbB(n)\! \subset\! \qty{0,1}^{n\times n}$ is the set of binary symmetric $n\times n$ matrices. 

For a probability distribution $\pr$ on $\Rd$ and measurable $f\push\pr$ is the pushforward measure satisfying $f\push\pr(A) = \pr(f\inv(A))$ for $A \subseteq \R^k$. We use standard asymptotic notation; ${a_n = O(b_n)}$ or $b_n = \Omega(a_n)$ if ${\limsup_n |a_n/b_n| \le C}$; $a_n = o(b_n)$ or $b_n = \omega(a_n)$ if ${\limsup_n \abs{a_n/b_n} = 0}$. 

\subsection{Differential Privacy for Graphs}

In the \eecdp{} framework, given a graph $\Gv = (V, E)$ with adjacency matrix $\Av \in \bbB(n)$, the sensitive records correspond the edges of the graph, $E$. The privacy mechanism, $\A_\e$, takes $\Av$ as input and produces a random output $\A_\eps(\Av)$ taking values in $\calZ$ such that
no single edge significantly affects the output.

\begin{definition}[\textbf{\eecdpe{}}]\label{def:edgedp}
    $\A_\e:\!\bbB(n)\!\to\!\cal{Z}$ satisfies $\e$-edge \dpy{} for $\eps > 0$ if, for all $\Av', \Av''\in \bbB(n)$ differing on a single edge, i.e., $\text{Ham}(\Av', \Av'') = 2$, and for all $S \subseteq \cal{Z}$,
    \begin{align}
        \label{eq:dp}
        \pr\qty( \A_\e(\Av') \in S ) \le e^\e \cdot \pr\qty( \A_\e(\Av'') \in S ).
    \end{align}
\end{definition}
Equivalently, from \cite{duchi2018minimax}, $\A_\eps$ defines a conditional probability distribution $\calQ_\eps$ such that for $\Zv = \A_\eps(\Av)$,
\begin{align}
    \calQ_\eps(\Zv \in S \,|\, \Av = \Av') \le e^\eps \cdot \calQ_\eps(\Zv \in S \,|\, \Av= \Av'').
\end{align}
In the local DP setting, the privacy mechanism is applied to \textit{each edge independently}; see, e.g., \citep{hehir2021consistent,li2022network}

\begin{definition}[\textbf{\eeldpe{}}]\label{def:noninteractive-dp}
    The privacy mechanism $\A_\eps = \qty{\A^\eps_{ij}: 1 \le i < j \le n}$ is $\eps$-edge locally DP if, for each $i < j \in [n]$ and $\Zv_{ij} = \A_\eps(\Av_{ij}) \in \calZ$, the resulting conditional distributions $\qty{\calQ^\eps_{ij}: 1 \le i < j \le n}$ are such that: for all $S \subset \calZ$ and $x, x' \in \qty{0, 1}$,
    \smallskip
    \begin{align}
        \calQ_{ij}(\Zv_{ij} \in S \,|\, \Av_{ij} = x) \le e^\eps \cdot \calQ_{ij}(\Zv_{ij} \in S \,|\, \Av_{ij} = x').
    \end{align}
\end{definition}

The privacy mechanism we focus on in this work is the symmetric \eflip{} \cite[see, e.g.,][]{mulle2015privacy,karwa2017sharing,qin2017generating,imola2020locally,hehir2021consistent,seif2022differentially,eden2023triangle}, $\M_\e : \bbB(n) \to \bbB(n)$, which produces an \eedpe{} \textit{synthetic graph}. This is closely related to Warner's randomized response \citep{warner1965randomized}. The following definition of \eflip{} is due to \cite{imola2020locally} and adapted from \cite[Definition~3.5]{hehir2021consistent}.

\begin{definition}[\eflip{}]\label{def:eflip}
    Given $\e > 0$, let\hfill\ \linebreak $\pi(\e)\defeq {1}/({e^\e+1})$ and $\F : \{0, 1\} \to \{0, 1\}$ be given by
    \begin{align}
    \F(x) = \begin{cases}
        x & \text{w.p. } 1-\pi(\e)\\
        1-x & \text{w.p. } \pi(\e)
    \end{cases}
    \end{align}
    For each $i \in [n]$, let $\Av_{i, *}$ be the edges connected to vertex $v_i$, and define the randomized response mechanism
    \begin{align}
        \M_{i, \eps}(\Av_{i, *}) = \left[ \zerov[i] \,\vert\, \F(\Av_{i, i+1}) \,\vert\, \F(\Av_{i, i+2}) \,\vert\, \cdots \,\vert\, \F(\Av_{i, n}) \right],
    \end{align}
    and the full randomized response for all vertices,
    $$
    T_\eps(\Av) = \left[ \M_{1, \eps}(\Av_{1, *}) \,\vert\, \cdots \,\vert\, \M_{n, \eps}(\Av_{n, *}) \right]\tr.
    $$
    Then \eflip{}, ${\M_\e: \bbB(n) \to \bbB(n)}$, is given by 
    $$
    \M_\e(\Av) = T_\e(\Av) + T_\e(\Av)\tr.
    $$
\end{definition}
By the closure property of DP mechanisms under post-processing, since \eflip{} satisfies $\e$-edge DP any analysis performed on $\M_\e(\Av)$ after \eflip{} will also satisfy the same privacy guarantee \citep{dwork2006calibrating}. Given a privacy budget $\e > 0$, and the edge flip probability in Definition~\ref{def:eflip}, we define 
\begin{align}
    \se &\defeq \sqrt{1-2\pi(\e)} = \sqrt{{e^{\e} - 1}\big /{e^{\e} + 1}},\\
    \te &\defeq \sqrt{\pi(\e)}  = \sqrt{{1}\big /{e^{\e} + 1}}.\label{eq:se-te}
\end{align}

\vspace*{-0.5em}
\subsection{Spectral Embedding and GRDPGs}
\label{sec:rdpg}

Statistical inference on graphs often begins by embedding the vertices in Euclidean space. Of numerous ways to achieve this \cite[e.g.,][]{belkin2003laplacian,perozzi2014deepwalk,grover2016node2vec}, we consider the adjacency spectral embedding of $\Av$ \cite{athreya2017statistical,rubin2022statistical}.

\begin{definition}
    \label{def:ase}
    For $\Av \in \bbB(n)$ and fixed $d<n$, let
    \begin{align}
        \Av = \Uv\ \! \L\pa{\Av} \ \! \Uv\tr + \Vv \ \! \boldsymbol{\Omega}(\Av) \ \! \Vv\tr,\label{eq:ase}
    \end{align}
    be the spectral decomposition of $\Av$, where $\abs{\L} \equiv \abs{\L(\Av)}$ are the top-$d$ eigenvalues by magnitude and $\Uv \in \bbU(n, d)$ are the corresponding eigenvectors. Then the adjacency spectral embedding of $\Av$ is given by
    \begin{align}
        \Xhat=\spec(\Av; d) = \Uv\abs{\L(\Av)}^{1/2} \in \R^{n \times d}.\label{eq:xhat-spec}
    \end{align}
\end{definition}

The spectral embedding is unique only up to orthogonal transformations; indeed, for any $\Qv \in \o(d)$, $\Uv' = \Uv\Qv \in \bbU(n,d)$ and leaves \cref{eq:ase} invariant. Moreover, if $\L$ contains $p$ positive and $q$ negative eigenvalues, then $\L = \abs{\L}^{1/2}\I[p,q]\abs{\L}^{1/2}$. GRDPGs can be seen as probabilistic graphical models which explicitly account for indefinite $\L$ in the spectral decomposition. 

To this end, a probability distribution $\pr$ on $\R^d$ is said to be $(p,q)$-admissible if $\supp(\pr)$ is compact and $0 \le \xv\tr\I[p,q]\yv \le 1$, for all $\xv, \yv \in \supp(\pr)$. In other words, samples from $\pr$ produce valid probabilities w.r.t. the indefinite inner product.

\begin{definition}[\textbf{GRDPG}]\label{def:grdpg}
    Let $\pr$ be $(p,q)$-admissible and let $\rho_n\le 1$. The random graph $\Av \in \bbB(n)$ is a generalized random dot-product graph with signature $(p,q)$, sparsity $\rho_n$, and latent positions $\X \in \Rnd$ if
    \begin{enumerate}[labelindent=!, itemindent=*, leftmargin=0em, label=\textup{(\roman*)}]
        \item The rows of $\X$ are observed \iid{} from $\pr$, and
        \item $\Av_{ij}\,|\,\X \sim_{\iid{}} \textup{Ber}(\rho_n\!\cdot\!{X_i\!\tr\I[p,q]X_j})\;\;\forall1\le i < j\!\le\!n$.
    \end{enumerate}
    More succinctly, $(\Av, \X) \sim \grdpg(\pr, \rho_n; p, q)$. Conditional on $\X \in \Rnd$, we write $\Av \sim \grdpg(\X, \rho_n; p, q)$.
\end{definition}

\noindent Some examples of GRDPGs include the following.
\begin{enumerate}[label=\textup{(\textbf{Ex}${}_\arabic*$)}, labelindent=0.5em, itemindent=*, leftmargin=0em,ref=\textup{(\textbf{Ex}${}_\arabic*$)}]
    \vspace*{-0.75em}
    \item\label{subremark:er} When $\pr = \delta_{\xv}$ is a Dirac mass at point $\xv \in \R^d$ with ${\pi \defeq \norm{\xv} \in (0, 1)}$, then $\grdpg(\pr, \rho_n;d,0)$ is an \er{} graph with probability $\pi\rho_n$. When $\norm{\xv} = 1$ and $\rho_n = \log{n}/n$, the resulting graph corresponds to the sharp connectivity threshold for \er{} graphs.
    \item\label{subremark:sbm} Let $\pr = \half(\delta_{\xv_1} \!+\! \delta_{\xv_2})$ where, for $\gamma>0$,
    \begin{align}
        \norm{\xv_1}=\norm{\xv_2}\! =\!\sqrt{1+\gamma}, \;\text{and}\\ \theta = \angle(\xv_1 \xv_2) = \arccos( \tfrac{1-\gamma}{1+\gamma} ).
    \end{align}
    Then $(\Av, \Xv) \sim \grdpg(\pr, \rho_n; 2, 0)$ is a stochastic block model $\sbm(n; \gamma, \rho_n)$. The intra- and inter-community edge probabilities are $\rho_n(1+\gamma)$ and $\rho_n(1-\gamma)$. On the other hand, when $(p,q)=(1,1)$, the SBM is disassortative.

    \item\label{subremark:dc-sbm} If, instead, $\xv_1 = \beta_1\xv$ and $\xv_2 = \beta_2 \xv$ for a fixed $\xv \in \R^d$ and $\beta_1, \beta_2 > 0$, then $\grdpg(\pr, \X, \rho_n; 2, 0)$ is a DC-SBM. On the other hand, if $\pr = \alpha\delta_{\xv_1} + (1-\alpha)\delta_{\xv_2}$ for $\alpha \in (0, 1)$, then the SBM has unbalanced communities.
\end{enumerate}

\begin{remark}\label{remark:grdpg-examples}
    GRDPGs are invariant to $\o(p,q)$ transformations and equivariant w.r.t. scaling transformations. To see this, let $\Qv \in \o(p,q)$ and $s > 0$ be a scale parameter. For $f(\xv) = \Qv(s \cdot \xv)$, it follows that 
    \begin{align}
        \vspace*{0.5em}
        \rho_nf(\X)\I[p,q]f(\X)\tr = s^2\rho_n\X \I[p, q] \X\tr,
        \label{eq:grdpg-equiv}
        \vspace*{0.5em}
    \end{align}
    and, $\grdpg(\pr, \rho_n) \distas \grdpg(f\push\pr, \rho_n / s^2)$.

    In other words, the latent positions of $\grdpg(\pr, \rho_n; p, q)$ are identifiable only up to an $\o(p, q)$ transformation, and $\rho_n$ is identifiable only up to scaling by a constant factor. In order to address the first invariance, we consider metrics on $\Rnd$ which are $\o(p, q)$ invariant (see \cref{def:dttinf}). For the second source of non-identifiability, we fix the scale of $\pr$ to satisfy
    \begin{align}
        \vspace*{0.5em}
        \E(\xi_1\tr\I[p,q]\xi_2)=1 \qq{for} \xi_1, \xi_2 \sim \pr.\label{eq:scale-identifiability}
        \vspace*{0.5em}
    \end{align}
    See, e.g., \cite[Corollary~2]{agterberg2020nonparametric} and \cite[line after Eq.~(2)]{lunde2023subsampling}, where the same condition is used to fix the scale of $\pr$.
\end{remark}
\pagebreak


\section{Main Results}
\label{sec:results}
This section presents the main results, and the proofs for the main results are collected in \cref{sec:proofs}.


\subsection{Closure of GRDPGs under \eflip{}}
\label{sec:edgeflip}

The key insight in the remaining sections comes from the following result, which states that the class of GRDPGs is closed under \eflip{}.

\begin{figure}[t]
    \centering
    \vspace*{-1.1em}
    \includegraphics[width=0.85\linewidth]{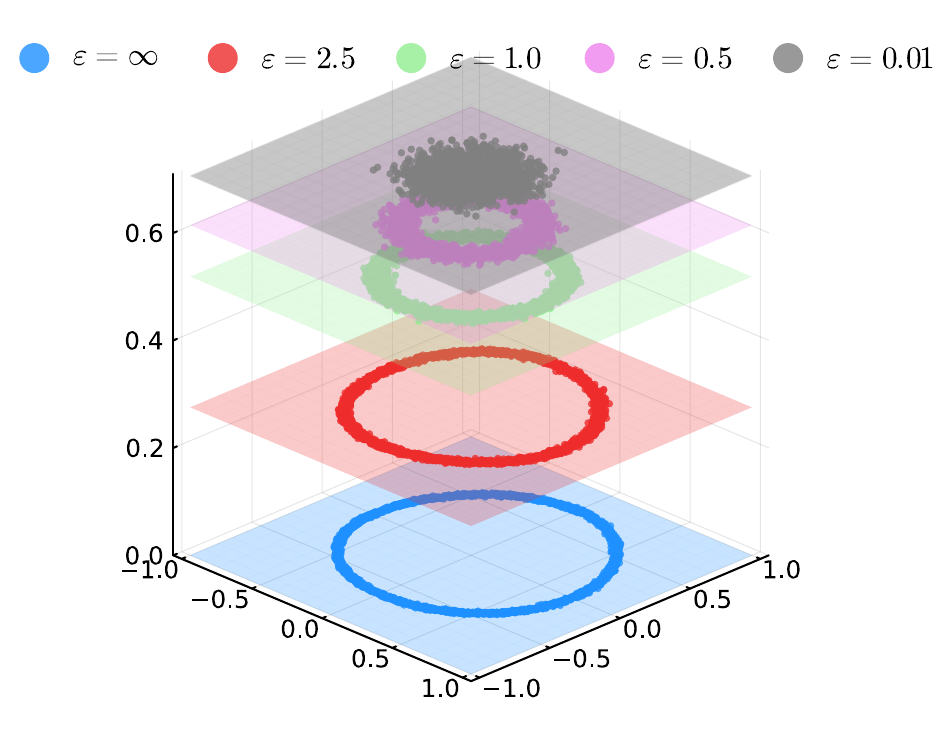}
    \vspace*{-1.2em}
    \caption{Illustration of \cref{prop:closure}. Spectral embedding of $\M_\e(\Av)$ after \eflip{} when $\X\!\sim\!\pr\!=\!\!\textup{Unif}(\bbS^1)$ and $\rho_n\!\equiv\!1$.}
    \label{fig:eflip-circle}
    \vspace*{-0.5em}
\end{figure}

\begin{theorem}
Suppose $(\Av, \Xv) \sim \grdpg(\pr, \rho_n; p, q)$, and for $\e\ge 0$, let $\M_\e(\Av)$ denote the \eedpe{} graph under \eflip{}. Let $\varphi_\eps: \Rd \to \R^{d+1}$ be the map given by $\varphi_\eps(\xv) = \te \oplus \se\rho_n^{1/2}\xv$, i.e.,
\begin{align}
    \varphi_\eps(\xv) 
    &= \begin{pmatrix} \te \\[0.5em] \se\rho_n^{1/2}\xv \end{pmatrix} \in \R^{d+1}.\label{eq:eflip-embedding}
\end{align}
Then, $(\M_\e(\Av), \varphi_\eps(\X)) \sim \grdpg\pa{\varphi_\eps{}\push\pr, 1; p+1, q}$.
\label{prop:closure}
\end{theorem}

Put simply, when \eflip{} is applied to a GRDPG with signature $(p, q)$, the resulting random graph is also a GRDPG with signature $(p+1, q)$, and its latent positions are embedded into $\R^{d+1}$ via $\varphi_\e$. See \cref{fig:eflip-circle} for an illustration. $\varphi_\eps$ scales the latent positions by a factor $\se$; as $\eps$ decreases, $\se$ decreases. This results in the squeezing effect in \cref{fig:eflip-circle}. Notably, $\M_\e(\Av)$ loses the sparsity pattern of $\Av$ whenever $\te = \omega(\rho^{1/2}_n)$, i.e., $\eps = o(\log{1/\rho_n})$. Nevertheless, $\M_\e(\Av)$ still lends itself to important inference tasks, as we will see in the coming sections.

The addition of the extra dimension in \cref{prop:closure} be viewed as the extra source of randomness introduced via \eflip{}. To see this, let $p_{ij} = X_i\tr\I[p,q]X_j$, and recall from \cref{eq:dp} that $\Av' = \M_\e(\Av)$ is such that
\begin{align}
    \pr( \Av'_{ij} = 1 ) = \pi(\e) (1-p_{ij}) + (1\!-\!\pi(\e)) p_{ij},
\end{align}
i.e., the probability of an edge $\Av'_{ij}$ is a convex mixture of the original connection probabilities of $\Av$. \cite[Appendix~B]{rubin2022statistical} establishes that GRDPGs are, essentially, the only random graph model which reproduce mixtures of connection probabilities as convex combinations in the latent space. Under \eflip{}, this manifests as a convex combination of the original latent positions $\X$ (when $\e=\infty$) and a single point in the extra $(d+1)$th dimension (when $\eps=0$), which corresponds to the latent position of an \er{} graph.

\begin{corollary}
    When $\e=0$ in \cref{prop:closure}, ${\varphi_\eps(\xv)\!\equiv\!(\sqrt{1/2}, \zerov[d]\tr)\tr}$ and ${\M_\e(\Av)\!\sim\!\text{\er}({1}/{2})}$.
\end{corollary}

By the preceding discussion, it would seem as though applying $\M_\e$ successively with $\e_1, \e_2$ would result in a GRDPG with signature $(p+2, q)$---one extra dimension for each application of $\M_\e$. However, probabilistically speaking, applying \eflip{} with probabilities $\pi(\e_1), \pi(\e_2)$ is equivalent to performing a single \eflip{} with adjusted probability 
$$
\pi= \pi(\e_1)(1-\pi(\e_2)) + (1-\pi(\e_1))\pi(\e_2).
$$ 
The next result addresses this by showing that the geometric interpretation is consistent with the probabilistic intuition.

\begin{proposition}[paraphrased; see \cref{sec:composition}]\label{prop:composition}
    Let $\Av' = \M_{\e_2} \circ \M_{\e_1}(\Av)$ and $\varphi: \R^d \to \R^{d+2}$ where
    $$
    \varphi(\xv) = \qty\big(\tau_2,\: \s_2\tau_1,\: \s_2\s_1\xv\tr)\tr \in \R^{d+2},
    $$
    where $\tau_i = \tau(\eps_i)$ and $\s_i = \s(\eps_i)$ for $i=1, 2$. Then, there exists $\Qv \in \o(p+2, q)$ and $a, b > 0$ such that
    \begin{align}
        \Qv\varphi(\xv) = (0, a, b\xv\tr)\tr \quad \forall \xv \in \R^d.
    \end{align}
\end{proposition}
\noindent By invariance under $\o(p+2,q)$ transformations from \cref{eq:grdpg-equiv}, \cref{prop:composition} shows there is a minimal $\R^{d+1}$ representation of the latent positions that encodes the GRDPG structure after multiple applications of \eflip{} which faithfully only has one extra dimension.


\subsection{Minimax rates for GRDPGs}
\label{sec:minimax}
Before considering the general problem of estimation from GRDPGs, it is instructive to first consider how \cref{prop:closure} simplifies analysis by considering the case when $\Av \sim \sbm(n; \gamma, \rho_n)$. Recall from \ref{subremark:sbm} that the latent positions are $\xv_1, \xv_2\in \Rd$ with $\norm{\xv_1}=\norm{\xv_2}=\sqrt{1+\gamma}$ and $\theta = \arccos(1-\gamma/1+\gamma)$. From \cref{prop:closure}, $\M_\e(\Av)$ is also an SBM with latent positions $\yv_i = \varphi_\e(\xv_i)$; see \cref{fig:sbm-eflip}. The intra- and inter-community edge probabilities become
\begin{align}
    a_\e = \yv_i\tr\I[d+1]\yv_i &= \te^2 + \se^2\rho_n(1+\gamma),\\
    b_\e = \yv_1\tr\I[d+1]\yv_2 &= \te^2 + \se^2\rho_n(1-\gamma).
\end{align}

For small $\eps \le 1$, an application of \cite[Theorem~2.1]{zhang2016minimax} shows that exact recovery is possible only when
\begin{align}
    \frac{(a_\e - b_\e)^2}{a_\e} > \frac{2\log{n}}{n} \!\implies\! \gamma \rho_n^2\se^4 = \Omega\qty( \frac{\log{n}}{n} ).\label{eq:sbm-lower-bound}
\end{align}
In the absence of privacy, however, the detection threshold for exact recovery requires 
\begin{align}
    \gamma\rho_n = \Omega\qty(\frac{\log{n}}{n}).
\end{align}
Compared to \cref{eq:sbm-lower-bound}, the nature of \eflip{} leads to an amplification of the effect due to sparsity $\rho_n$, and coincides with the observations in \cite[p.~13]{hehir2021consistent} for \textit{weak recovery}. A similar analysis can be carried out for \eedpe{} exact recovery from SBMs with $k$-communities and for DC-SBMs using the results from \cite{gao2018community}.

\begin{figure}[t]
    \centering
    \vspace*{-1em}
    \includegraphics[width=0.66\linewidth]{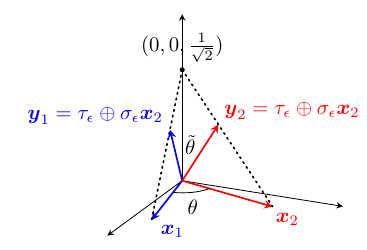}
    \vspace*{-0.7em}
    \caption{Illustration of \eflip{} for $\sbm(n; \gamma, \rho_n)$}
    \label{fig:sbm-eflip}
    \vspace*{-1em}
\end{figure}

We now turn our attention to the problem of estimating the latent positions of GRDPGs under \eedpe{}. Since the latent positions are identifiable only up to $\o(p,q)$ transformations by \cref{eq:grdpg-equiv}; any metric for assessing the uniform error in the latent positions of GRDPGs must take into account this non-identifiability. To this end, \cite{yan2023minimax} define the metric $\dttinf$ as follows.

\begin{definition}[$\dttinf$ metric]\label{def:dttinf}
For $\X, \Y \in \Rnd$, 
\medskip
\begin{align}
    \dttinf({\X, \Y}) = \min_{\Ov \in \o(d) \cap \o(p,q)} \ttinf{\Y\Qy\inv - \X\Qx\inv\Ov},
\end{align}
where $\Qx, \Qy \in \o(p,q)$ are the spectral alignments such that $\X\Qx\inv\! =\!\U_\X \abs{\L(\X)}^{1/2}$, and similarly for $\Y$.
\end{definition}

For fixed $p, q, d$, let $\calX \subset \Rd$ be the set of all regular $(p, q)$-admissible latent positions:
\begin{align}
    \calX \defeq \Big\{\X \in \Rnd: \;& \X\I[p,q]\X\tr \in [0,1]^{n\times n} \Big\}.
\end{align}
For $\Av \sim \grdpg(\X, \rho_n; p, q)$, let $\hX: \bbB(n) \to \Rnd$ be an estimator of the latent positions $\X \in \calX$. The minimax risk for estimating the latent positions $\X$ is given by
\begin{align}
    \rn(\calX) = \inf_{\hX} \sup_{\X \in \calX } {\E}_{_{\Av \sim \grdpg(\X, \rho_n; p, q)}}\!\!\qty[\dttinf\qty(\hX(\Av), \X)],
\end{align}
where the infimum is taken over all estimators $\hX$. In the absence of privacy and when $n\rho_n=\omega(\log{n})$, the minimax rate from \cite[Theorem~1~\&~Corollary~1]{yan2023minimax} is
\begin{align}
    \rn(\calX) = \Omega\qty(\sqrt{\frac{\log{n}}{n\rho_n}} ).\label{eq:minimax-infty}
\end{align}
We point out that the rate in \cref{eq:minimax-infty} implicitly accounts for the $\rho_n^{-1/2}$ rescaling needed to reover the original latent positions $\X$ from $\Av \sim \grdpg(\X, \rho_n; p, q)$.

In the \eeldpe{} setting, let $\mathbb{A}_\eps$ denote the set of all $\eps$-edge LDP mechanisms satisfying \cref{def:noninteractive-dp}. Under the additional constraint of privacy, we restrict our attention to estimators $\hX: \calZ \to \Rnd$ which use the private output $\Zv = \A_\e(\Av) \in \cal{Z}$; in this setting the minimax risk is given by
\begin{align}
    \!\!\rn(\calX, \eps)  = \inf_{\A_\eps \in \mathbb{A}_\epsilon}\inf_{\hX} \sup_{\Xv \in \calX } {\E}\qty[\dttinf\qty(\hX(\A_\eps(\Av)), \X)],\label{eq:priv-minimax}
\end{align}
where the expectation is taken over the randomness in $\Av$ and in $\A_\eps \in \bbA_\eps$. The following result provides a lower bound on the minimax risk for estimating the latent positions of GRDPGs under \eeldpe{}.

\begin{theorem}\label{thm:minimax}
    Let $\calX$ be as defined above, and let $n\rho_n = \omega(\log{n})$. Then, for $0 < \eps < 3/8\rho_n$ and for sufficiently large $n$, the minimax risk satisfies
    \begin{align}
        \rn(\calX,\eps) = \Omega\qty(\sqrt{\frac{\log{n}}{n\se^4\rho_n^2}} ).\label{eq:minimax}
    \end{align}
\end{theorem}

In other words, \cref{thm:minimax} characterizes the worst-case risk any estimator of the latent positions of a GRDPG can achieve under \textit{all possible} \eeldpe{} mechanisms satisfying \cref{def:noninteractive-dp}. In the \textit{high-privacy regime} where $\eps \to 0$ as $n \to \infty$, the privacy-utility tradeoff under \eeldpe{} differs from the non-private setting in \cref{eq:minimax-infty} in two key ways:

\begin{enumerate}[label={(\roman*)}, itemsep=0.1em]
\item The effective sample size reduces from $n$ to $\sigma^4(\eps)n \ll n$. This phenomenon is well-known in \edpe{} estimation \citep[see, e.g.,][]{barber2014privacy,duchi2018minimax,rohde2020geometrizing,asoodeh2021local}.
\item Moreover, in contrast to \cref{eq:minimax-infty}, there is an amplification of the effect due to sparsity in \cref{eq:minimax}, i.e., $\rho \mapsto \rho_n^2$. In \cite{li2022network}, a similar phenomenon was observed for change-point localization in networks.
\end{enumerate}

Note that in order for the lower bound in \cref{eq:minimax} to not be vacuous, we require that $\se^4\rho_n^2 = \Omega(\log{n} / n )$. By noting that $\sigma(\eps)^2 = \tanh(\eps/2)$, this is equivalent to:
\begin{align}
    \eps = \Omega\qty(\tanh\inv\qty(\sqrt{\frac{\log{n}}{n\rho_n^2}})),\label{eq:feasibility}
\end{align}
In other words, if $\eps\to0$, consistent recovery of the latent positions is only possible for sufficiently dense graphs. For instance, while graphs with sparsity ${\rho_n \asymp (\log n/n)^{2/3}}$ admit consistent estimation in the non-private setting from \cref{eq:minimax-infty}, in the minimax sense no consistent estimators exist at the same sparsity under \eeldpe{} from \cref{eq:priv-minimax}. We also note that \cref{eq:feasibility} coincides with the detection threshold for exact recovery under \eflip{} when ${\grdpg(\pr, \rho_n; 2, 0)\distas \sbm(\gamma, \rho_n)}$ as noted in~\cref{eq:sbm-lower-bound}.


\subsection{Privacy-Adjusted Spectral Embedding}
\label{sec:convergence}

\begingroup
Having established the thresholds for estimating latent positions of GRDPGs under \eedpe{}, we now focus on constructing an optimal estimator. Without privacy constraints, for $(\Av, \Xv) \sim \grdpg(\pr, \rho_n; p, q)$ the spectral embedding $\hX = \spec(\Av; d)$ serves as an estimator for the latent positions after global scaling by the sparsity factor, i.e., $\rho_n^{1/2}\Xv$ \citep[Theorem~1]{rubin2022statistical}. In the regime where $\rho_n = o(1)$, $\hX$ estimates a quantity that shrinks to zero as $n$ grows. Rescaling by the sparsity factor---specifically, using $\rho_n^{-1/2}\hX$---yields a consistent estimator for $\X$ up to $\o(p,q)$ transformations \citep[Theorem~1]{rubin2022statistical} and \citep[Theorem~4]{agterberg2020nonparametric}. 

Under privacy constraints, however, a privacy mechanism $\A \in \mathbb{A}_\eps$ may alter $\Av$ in ways that the spectral embedding does not explicitly account for. Consequently, accurate statistical inference from the \eedpe{} graph requires choosing an appropriate privacy mechanism and compensating for its geometric distortion in the estimation procedure.

\begin{algorithm}[t]
    \caption{{\small Privacy-Adjusted Spectral Embedding}}\label{alg:pase}
    \SetAlgoLined
    \SetKwInput{KwInput}{Input}
    \KwInput{$\M_\e(\Av)$ and privacy budget $\e$}
    \BlankLine
    \begin{algorithmic}[1]
    \State Compute the privacy-adjusted adjacency matrix
    \vspace*{-1em}
    $$
    \Ac = \f{1}{\se[2]}\qty(\M_\e(\Av) - \te[2]\onev[n]\onev[n]\tr).
    $$
    \vspace*{-1em}
    \State Estimate the sparsity $\ch{\rho}_n$
    \vspace*{-1em}
    $$
    \ch{\rho}_n = {n \choose 2}\inv\sum_{i < j} \Ac_{ij}.
    $$
    \vspace*{-1em}
    \State Spectral embedding
    \vspace*{-0.5em}
    $$
    \Xc = \spec\qty\big(\Ac; d).
    $$
    \vspace*{-1em}
    \end{algorithmic}
    \Return{$\Xc \in \Rnd$ and $\ch{\rho}_n \in [0, 1]$}
\end{algorithm}%

To this end, we select $\A = \M_\eps$ as the \eflip{} mechanism and introduce a \textit{privacy-adjusted spectral embedding} procedure in \cref{alg:pase}. The intuition behind \cref{alg:pase} is as follows. By \cref{prop:closure}, the latent positions of $\M_\e(\Av)$ and $\Av$ are connected via the mapping $\varphi_\e$. Since the privacy parameter $\eps$ can be disclosed publicly without disclosure risk \cite{dwork2014algorithmic,dwork2019differential}, we can reverse the effect of $\varphi_\e$ on the spectral estimate by applying the inverse transformation to $\M_\e(\Av)$. The remaining challenge is that the sparsity parameter $\rho_n$ is not known publicly. 
We mitigate the influence of $\rho_n$ by computing an estimate $\ch{\rho}_n$ from the synthetic graph, as described in Step~2 of \cref{alg:pase}. The following result establishes that this procedure yields a consistent estimator for the latent positions $\Xv$, up to $\o(p,q)$ transformations.\pagebreak

\begin{theorem}\label{thm:convergence-rate}
    Suppose $(\Xv, \Av) \sim \grdpg\pa{\pr, \rho_n; p, q}$ where $\E(X_1\tr\I[p,q]X_2)=1$. For $\e > 0$ let $\M_\e(\Av)$ be the \eedpe{} graph under \eflip{}, and let $\Xc$ be the privacy-adjusted spectral embedding using \cref{alg:pase}. If $n\rho_n=\omega(\log{n})$, then with probability $1 - O(n^{-1})$,
    \begin{align}
        \dttinf\qty( \frac{\Xc}{\sqrt{\ch{\rho}_n}}, {\Xv} ) = O\qty(\frac{\log{n}}{\sqrt{ n\se^4 \rho_n^2 }}).\label{eq:convergence-rate}
    \end{align}
\end{theorem}

In light of \cref{thm:minimax}, we see that the convergence rate of the privacy-adjusted spectral embedding $\Xc$ matches the lower bound in \cref{eq:priv-minimax} up to an $O(\sqrt{\log{n}})$ gap. A similar discrepancy exists in the best-known bounds in the absence of privacy constraints.  In \cite[Theorem~1]{rubin2022statistical} and \cite[Theorem~4]{agterberg2020nonparametric}, the convergence rate for estimating %
$\X$ using $\hX = \spec(\Av; d)$ is
\begin{align}
    \dttinf\qty(\frac{\hX}{\sqrt{{\rho}_n}}, {\Xv}) = \op\qty( \frac{\log{n}}{\sqrt{ n\rho_n }} ).\label{eq:convergence-rate-no-privacy}
\end{align}
In contrast to \cref{eq:minimax-infty}, the rate in \cref{eq:convergence-rate-no-privacy} has precisely the same $O(\sqrt{\log{n}})$ gap. Therefore, modulo $\log{n}$ factors, for the problem of estimating the latent positions of GRDPGs under \eedpe{}, we see that:
\vspace*{-0.7em}
\begin{enumerate}[label={(\roman*)}, itemsep=0.0em]
    \item \eflip{} is an optimal privacy mechanism,
    \item the privacy-adjusted spectral embedding $\Xc/\ch{\rho}_n$ after \eflip{} is an optimal estimator for $\X$, and
    \item the range of parameters $\eps, \rho_n \in (0,1)$ satisfying ${\se^4\rho_n^2 = \widetilde{\Theta}(1/n)}$ determines a phase-transition boundary for consistent estimation.
\end{enumerate}
\vspace*{-0.7em}

We also point out that, as noted in the discussion preceding \cref{eq:scale-identifiability}, $\E(X_1\tr\I[p,q]X_2)=1$ is fixed in \cref{thm:convergence-rate} in order to ensure $\rho_n$ is identifiable and that Step~2 of \cref{alg:pase} produces a consistent estimate of $\rho_n$. If, instead, $\E(X_1\tr\I[p,q]X_2)=\mu$, then the results in both \cref{eq:convergence-rate} and \cref{eq:convergence-rate-no-privacy} will hold for $\X/\sqrt{\mu}$ in place of~$\X$. 
\endgroup

\begin{figure}[t]
    \centering
    \vspace*{-1.4em}
    \includegraphics[width=\linewidth]{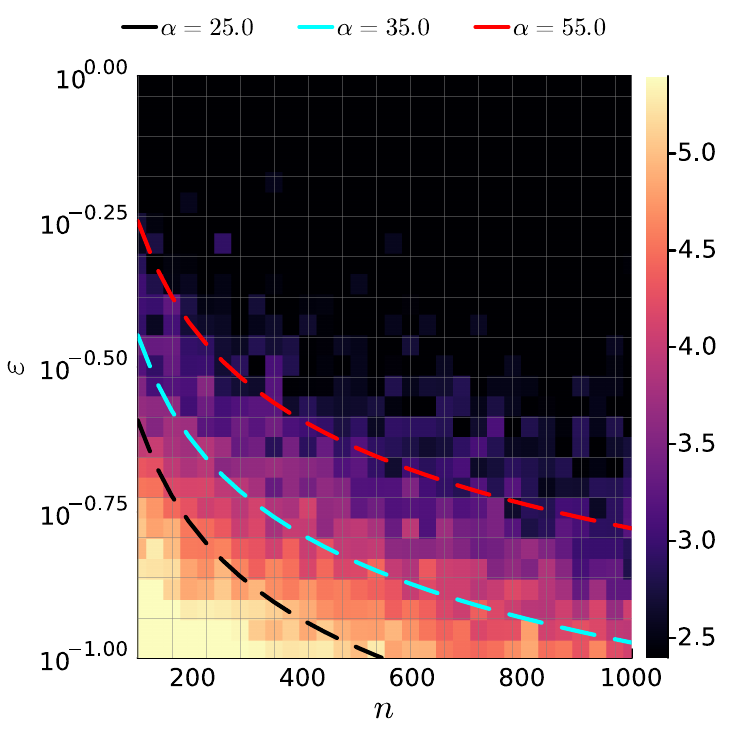}
    \vspace*{-1.75em}
    \caption{Illustration of the bounds in \cref{thm:convergence-rate}. Heatmap of $\dttinf(\Xc, \X)$ error for $\eps$ vs. $n$ in \cref{exp:convergence}.}
    \label{fig:threshold}
    \vspace*{-0.75em}
\end{figure}

\begin{experiment}\label{exp:convergence}
    In order to examine the bound in \cref{thm:convergence-rate}, we generate $(\Av, \Xv) \sim \grdpg(\pr, \rho_n; p, q)$ where $\pr = \textup{Unif}(S^1)$ is the uniform distribution on the circle, $\rho_n = \log^4{n}/\sqrt{n}$, and $p,q=1$. For each $\eps, n$, the \eedpe{} graph $\M_\e(\Av)$ is generated via \eflip{} and $\Xc$ is computed using \cref{alg:pase}. The error $\dttinf(\Xc, \X)$ is averaged across $10$ iterations. The results are shown in \cref{fig:threshold}. The shaded curves, given by $\s(\alpha\epsilon)^2 = \log{n}/\sqrt{n\rho_n^2}$ for $\alpha\in\qty{25, 35, 55}$, are contours where \cref{eq:convergence-rate} is expected to be constant. As \cref{fig:threshold} shows, the bound is fairly tight. 
\end{experiment}


\subsection{Topological inference under \eflip{}}
\label{sec:tda}

While the previous sections have focused on estimating the latent positions of a GRDPG under \eedpe{} in the $\dttinf$ metric, in this section we examine how $\M_\e(\Av)$ can be used to extract meaningful topological and geometric information underlying $\X$. To this end, we use tools from topological data analysis (TDA). Owing to length constraints, we provide an accessible overview of TDA in \cref{sec:tda-appendix}. The main insight for this section comes from the following result, which is a simple consequence of \cref{prop:closure}.
\begin{corollary}
    Under the conditions of Theorem~\ref{prop:closure}, the map $\varphi_\e$ is a homeomorphism, and $\varphi_\e(\X) \simeq \X$.
    \label{cor:homeomorphism}
\end{corollary}
In simpler terms, \cref{cor:homeomorphism} states that, in view of \cref{prop:closure}, the latent positions of $\M_\e(\Av)$ are related to those of $\Av$ by deformations which only allow stretching/squeezing/bending. The result is immediately obvious since $\varphi_\e: \Rd \to \R^{d+1}$ is an affine map. Notably, homeomorphisms preserve topological structures, as shown for spectral embeddings in \cref{fig:eflip-circle}.

Persistent homology forms the backbone of most TDA routines, where multiscale  geometric and topological  features underlying a collection of points $\X \in \Rnd$ are encoded in an object called a \textit{persistence diagram} $\dgm(\X)$. In a nutshell, $\dgm(\X)$ is a collection of points in the upper-half plane ${\Omega = \qty{(b, d) \in \R^2: b \le d}}$, where each point $(b, d) \in \dgm(\X)$ with $b < d$ represents the birth time, $b$, and death time, $d$, of a topological feature in $\X$. The topological features are indexed by the dimension of the feature, e.g., $0d$ features are connected components or `clusters', $1d$ features are loops, $2d$ features are holes, and so on. Given two persistence diagrams $D_1, D_2 \subset \Omega$, the \textit{bottleneck distance} $\w(D_1, D_2)$ quantifies the similarity between the topological features in the two diagrams. 

The following result shows that the persistence diagram $\dgm(\Xc/\ch{\rho}_n)$ is a consistent estimator of $\dgm(\X)$.
\begin{theorem}\label{thm:bottleneck-convergence}
    Under the conditions of \cref{thm:convergence-rate}, let $D_n = \dgm(\X)$ and $\ch{D}_n = \dgm(\Xc/\sqrt{\ch{\rho}_n})$. Then, with probability at least $1 - O(n^{-1})$,
    \begin{align}
        \Winf\qty(\ch{D}_n, D_n) = O\qty( \frac{\log{n}}{\sqrt{ n\se^4 \rho_n^2 }} ).\label{eq:bottleneck-convergence}
    \end{align}
\end{theorem}
\vspace*{-0.75em}
\cref{thm:bottleneck-convergence} guarantees that $\Xc$ consistently recovers the essential topological features of $\X$; this includes connected components, loops, and other salient geometric and topological features. Since the $0$th order persistence diagram precisely captures the clusters in the underlying data and by noting that the rate in \cref{eq:bottleneck-convergence} is the same as in \cref{eq:convergence-rate}, the privacy-adjusted spectral embedding $\Xc$ provides the same privacy-utility tradeoff guarantees as in \cref{sec:convergence} and can recover more general topological features of $\X$.

Although we do not pursue it in this work, the convergence in $d_{2, \infty}$ metric in \cref{thm:convergence-rate} can be used to show that ``flat clustering'' methods like $k$-means are consistent under \eeldpe{} when the latent positions form well-separated clusters (see, e.g., \cite[Theorem~6]{lyzinski2014perfect}). Similarly, we believe that it should be possible to establish the consistency of DBSCAN when the distribution generating the latent positions, $\pr$, is sufficiently regular, e.g., admits a H\"{o}lder density. In other words, \cref{thm:convergence-rate} could be used to show the consistency of other downstream inferential procedures under \eeldpe{} when necessary conditions are met.



\vspace*{-0.6em}
\section{Experiments}
\label{sec:experiments}
\vspace*{-0.4em}


\begin{figure*}[ht!]
    \centering
    \vspace*{-0.5em}
    \ \hspace*{-1.5em}
    \begin{subfigure}[c]{0.97\textwidth}
    \begin{subfigure}[t]{0.24\textwidth}
        \centering
        \includegraphics[height=1.1\textwidth]{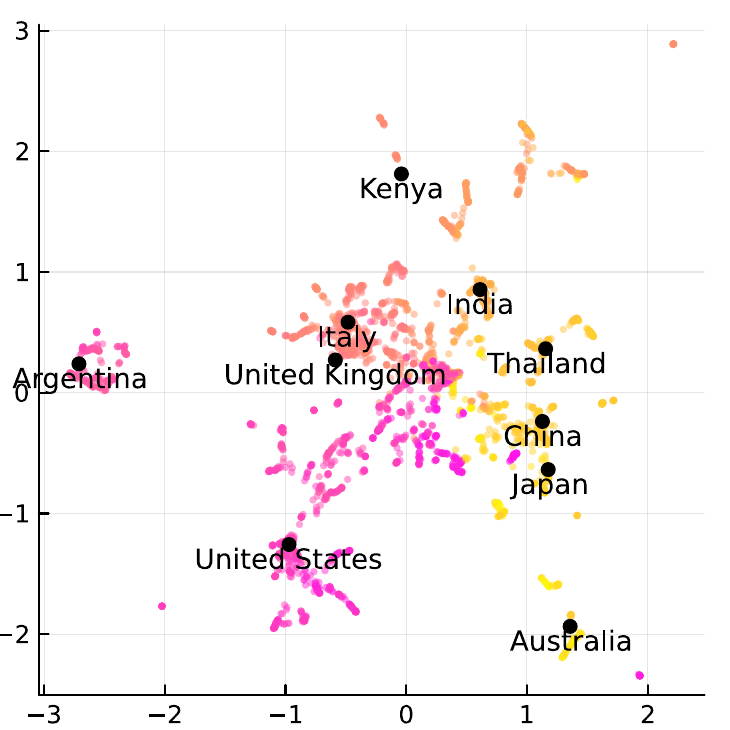}
        \caption{$\e=\infty$}
    \end{subfigure}
    \begin{subfigure}[t]{0.24\textwidth}
        \centering
        \includegraphics[height=1.1\textwidth]{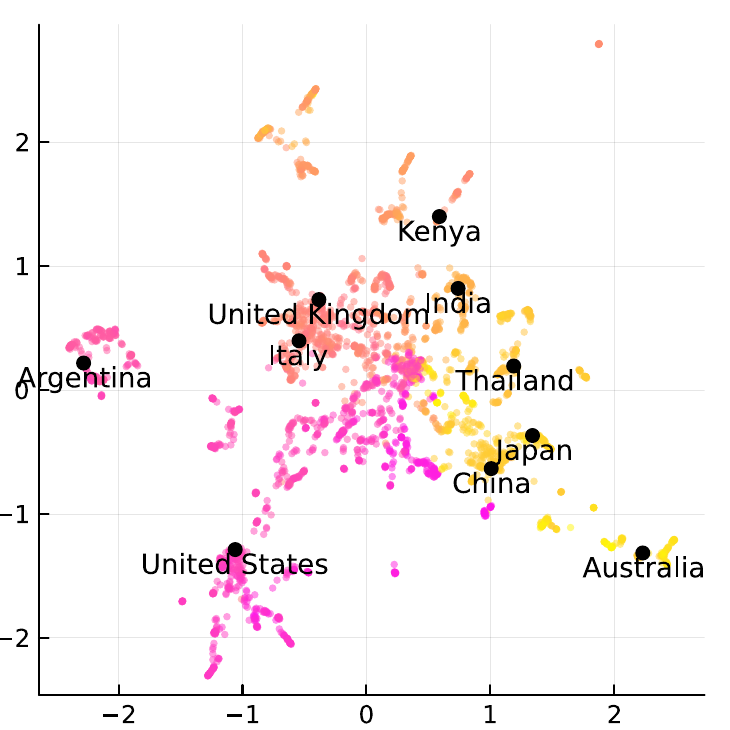}
        \caption{$\beta=0.33$}
    \end{subfigure}
    \begin{subfigure}[t]{0.24\textwidth}
        \centering
        \includegraphics[height=1.1\textwidth]{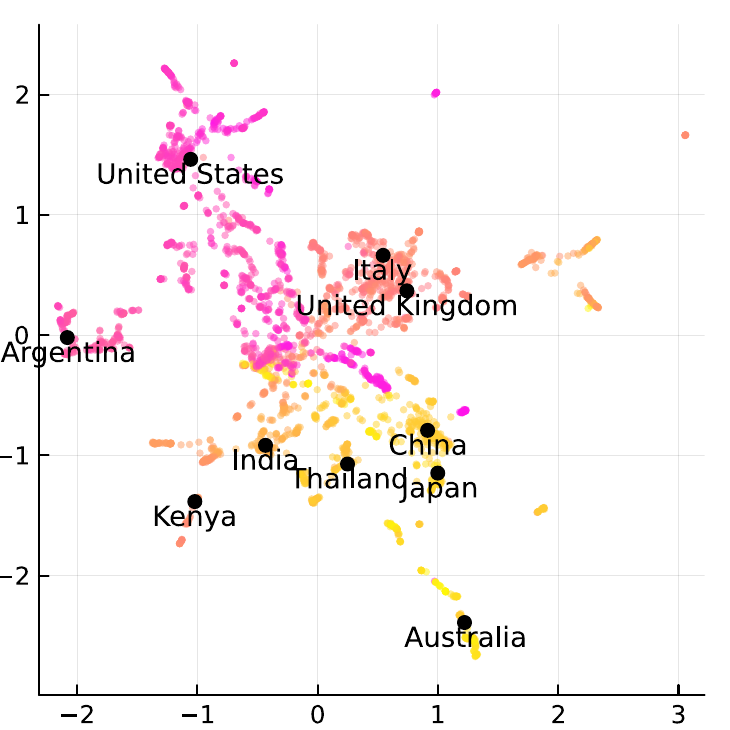}
        \caption{$\beta=0.66$}
    \end{subfigure}
    \begin{subfigure}[t]{0.24\textwidth}
        \centering
        \includegraphics[height=1.1\textwidth]{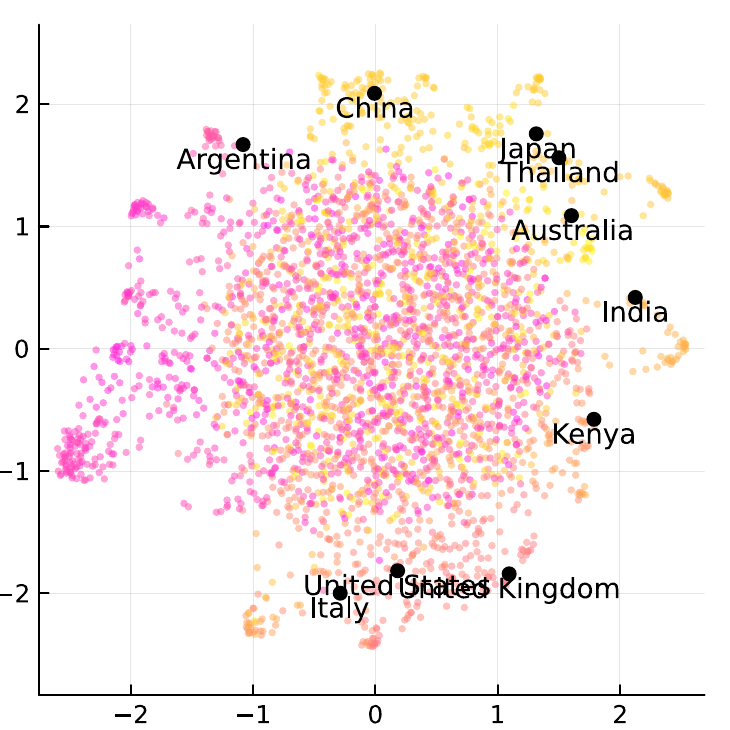}
        \caption{$\beta=1.1$}
    \end{subfigure}
    \end{subfigure}%
    \begin{subfigure}[c]{0.025\textwidth}
        \includegraphics[height=12\textwidth]{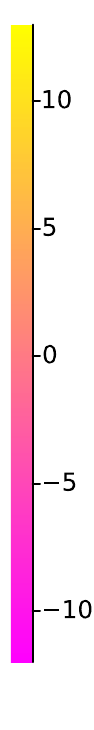}
    \end{subfigure}
    \vspace*{-0.4em}
    \caption{Spectral embedding followed by UMAP for the OpenFlights network in \cref{exp:flights}.}
    \label{fig:flights}
    \vspace*{1em}
\end{figure*}

\begin{experiment}\label{exp:flights}
    To examine recovery of latent positional information in real-world data, we use the OpenFlights dataset 
    where $\Av$ represents direct flight connections between $n=3254$ airports. Figure~\ref{fig:flights}\,(a) displays the scatterplot of $\mathfrak{P}(\X)$, obtained using spectral embedding of $\Av$ followed by UMAP---which uses topological information to perform unsupervised dimension reduction \citep[Section~2]{mcinnes2018umap2}. Each airport $\qty{i}$ is colored based on its timezone. Although the unweighted nature of the graph implies that all connections are treated equally, and disregards any additional information such as distance, flight duration, or frequency---$\mathfrak{P}(\X)$ recovers the subtle geographic information associated with the airports, e.g., Japan is close to China and Italy is close to the United Kingdom. 
    
    \vspace*{-0.4em}
    \noindent Since $\se^2 \approx \eps$ in the admissible regime of \cref{eq:bottleneck-convergence}, taking $\rho_n$ to be the sparsity of $\Av$, $\beta \in \qty{\tfrac{1}{3}, \tfrac{2}{3}, \tfrac{11}{10}}$, and
    $$
    \eps_n^2 = {\log{n}/\rho_n^2 n^\beta},
    $$
    Figures~\ref{fig:flights}\,(b,c,d) show the plots of $\mathfrak{P}(\Xc)$ obtained using $\M_\e(\Av)$. For moderately small $\eps$,
    Figures~\ref{fig:flights}\,(b,c) show that $\mathfrak{P}(\Xc)$ preserves (to a large extent) the salient relationships in Figure~\ref{fig:flights}\,(a). However, for very small $\eps$, as shown in Figure~\ref{fig:flights}\,(d),  $\mathfrak{P}(\Xc)$ no longer captures this information.
\end{experiment}


\begin{figure*}[!ht!]
    \centering
    \begin{subfigure}[t]{0.3\textwidth}
        \includegraphics[width=\textwidth]{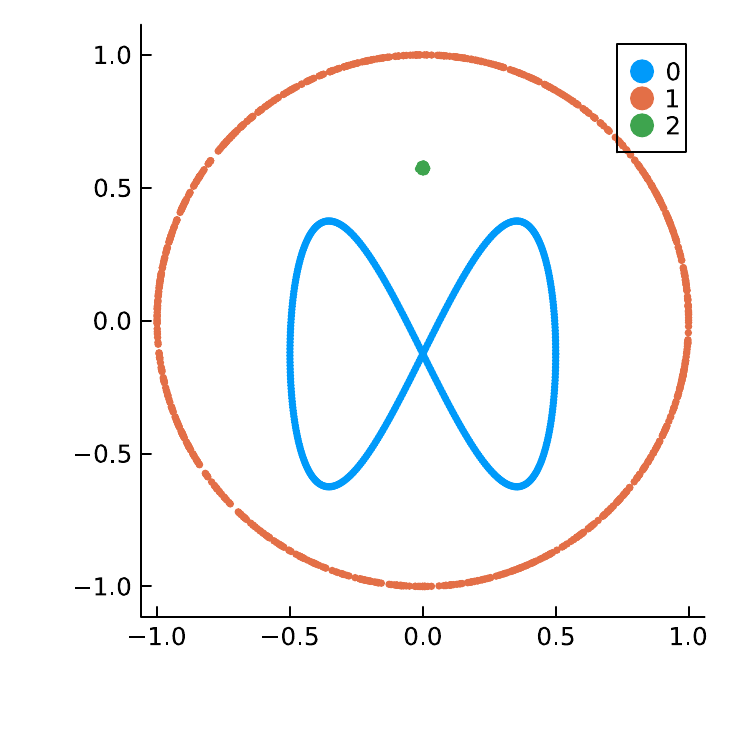}
        \vspace*{-2.0em}
        \caption{Latent space}
    \end{subfigure}
    \begin{subfigure}[t]{0.3\textwidth}
        \includegraphics[width=\textwidth]{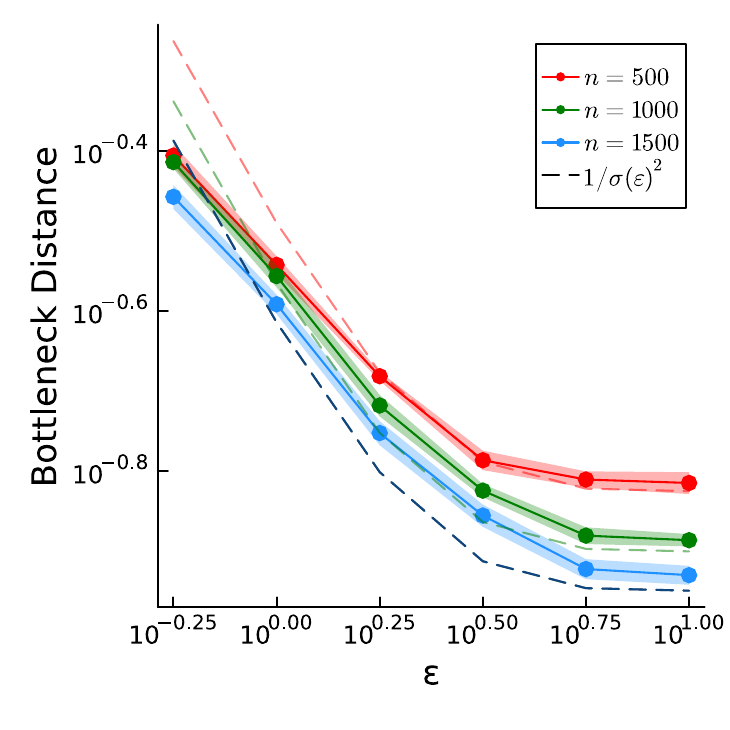}
        \vspace*{-2.0em}
        \caption{Bottleneck distance}
    \end{subfigure}
    \begin{subfigure}[t]{0.3\textwidth}
        \includegraphics[width=\textwidth]{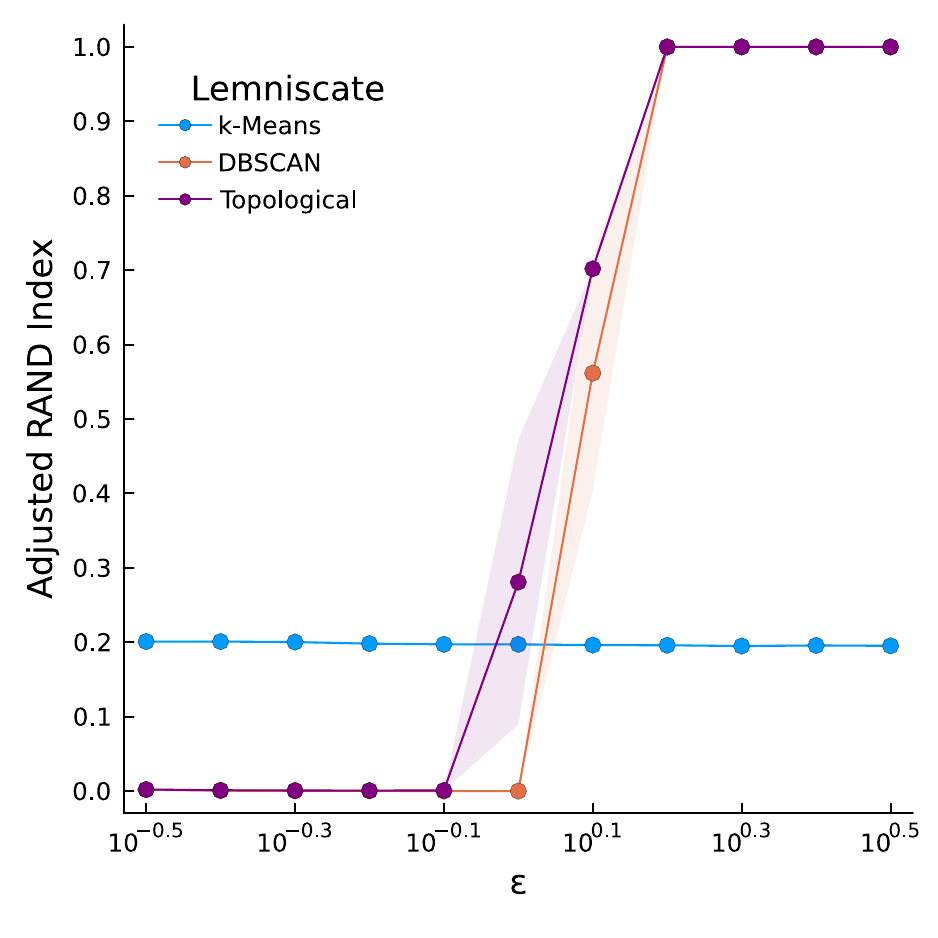}
        \vspace*{-2.0em}
        \caption{Clustering Accuracy}
    \end{subfigure}
    \vspace*{-0.33em}
    \caption{Recovery of topological information for the setup in \cref{exp:lemniscate}.}
    \label{fig:lemniscate}
    \vspace*{1em}
\end{figure*}

\begin{experiment}\label{exp:lemniscate}
    In this experiment, we illustrate the benefit of the topological perspective by employing persistence diagrams to perform \eedpe{} community detection. We generate $N=4n$ points for $\X$ from the following shape: (i) $2n$ are sampled uniformly from a circle, (ii) $n$ points sampled from a Lemniscate (which looks like "$\infty$"), and (iii) $n$ points sampled from a cluster inside the circle but disjoint from the rest. The shape contains $3$ non-trivial connected components, and has three loops (order-$1$ homological features): one from the circle, and two from the Lemniscate. See~\cref{fig:lemniscate}~(a). We then generate $\Av \sim \grdpg(\X, 1; 2, 0)$.

    For a range of $\eps$, we obtain the persistence diagram $\ch{D}^\e$ from $\M(\Av)$ and compute the bottleneck distance $W_\infty(D, \ch{D}^\e)$. The results are averaged across $10$ iterations. \cref{fig:lemniscate}~(b) shows the convergence in the bottleneck distance for $n \in \qty{500, \dots, 1500}$, and the black dashed line plots the r.h.s. of \cref{eq:bottleneck-convergence} when $n=1500$.

    Clustering for $\X$ by itself is particularly challenging for algorithms like $k$-means since $\X$ doesn't admit ``flat'' clusters. To address this limitation, we use a topological clustering algorithm (which is a simplified variant of the algorithm in \cite{kurlin2016fast}) and is described in Algorithm~2 in \cref{sec:tda-appendix}. Figure~\ref{fig:lemniscate}\,(c) plots the Adjusted RAND Index between the true labels and the predicted labels obtained using: (i) the topological approach, (iii) $k$-means when $k=3$ is provided as input, and (iii) using DBSCAN with \texttt{minPts} to be the size of the smallest cluster. The results illustrate how \cref{thm:bottleneck-convergence} enables access to using more tailored topological approaches to recover the underlying clusters.
\end{experiment}


\begin{figure*}[!h!]
    \centering
    \vspace*{-1em}
    \begin{subfigure}[t]{0.3\textwidth}
        \includegraphics[width=\textwidth]{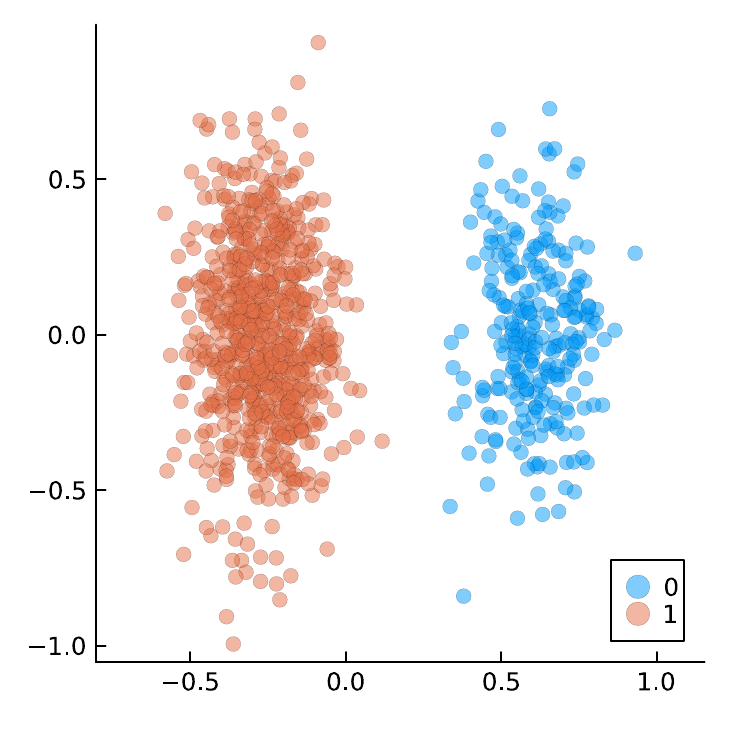}
        \vspace*{-2.0em}
        \caption{Latent space}
    \end{subfigure}
    \begin{subfigure}[t]{0.3\textwidth}
        \includegraphics[width=\textwidth]{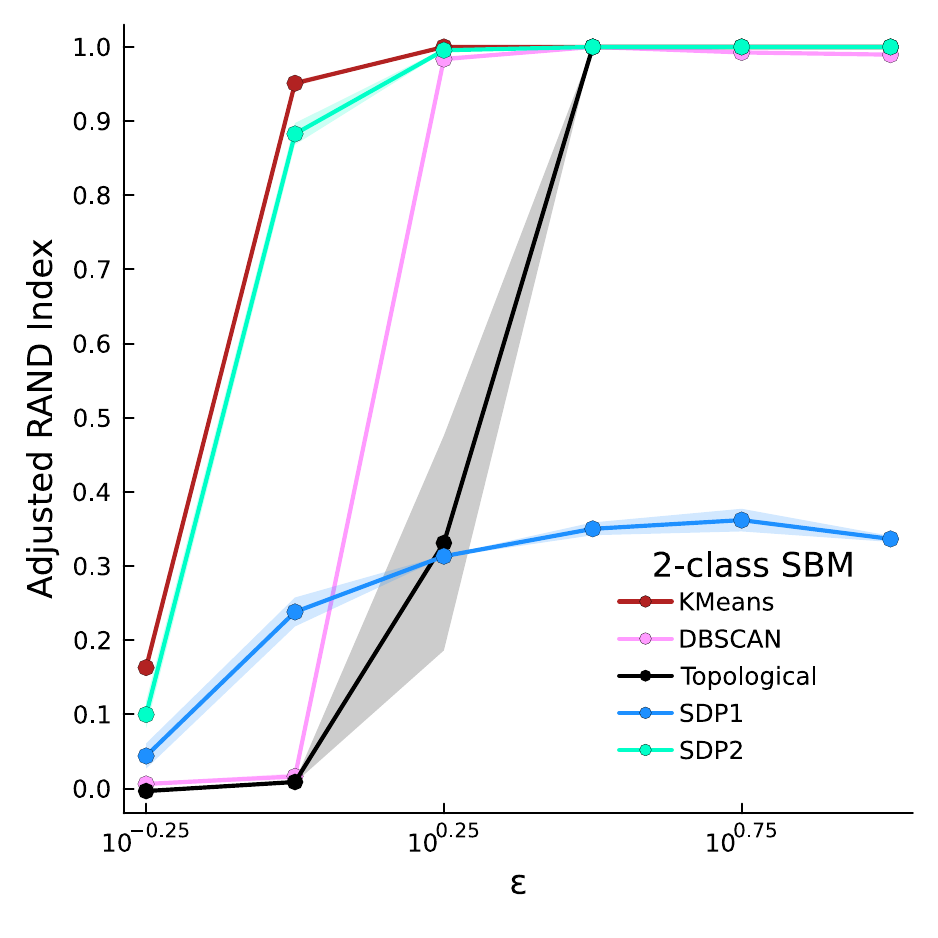}
        \vspace*{-2.0em}
        \caption{Clustering Accuracy}
    \end{subfigure}
    \begin{subfigure}[t]{0.3\textwidth}
        \includegraphics[width=\textwidth]{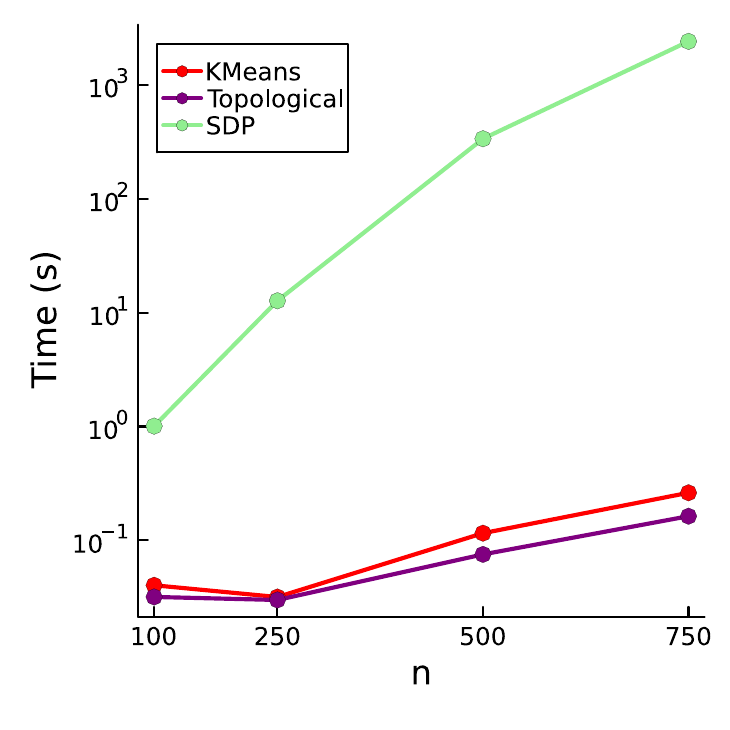}
        \vspace*{-2.0em}
        \caption{Clustering Accuracy}
    \end{subfigure}
    \vspace*{-0.33em}
    \caption{Performance of spectral methods alongside SDP-based methods in \cref{exp:clustering}.}
    \label{fig:clustering}
    \vspace*{-0.5em}
\end{figure*}

\begin{experiment}\label{exp:clustering}
    It is well-known that spectral estimators are suboptimal for several inferential procedures at the information-theoretic thresholds for sparse graphs. Many  semi-definite programming (SDP) methods have been proposed in literature to address this limitation for specific models \cite{abbe2017community}. Nevertheless, this approach comes at the expense of computational time and a careful selection of hyperparameters. To benchmark the performance, we generate $\Av$ from an imbalanced $2$-class SBM with imbalance ratio $1/3$. See \cref{fig:clustering}\,(a).
    
    SDP-solvers are particularly well suited for SBMs. We use two SDP-approaches: (i) SDP$_1$ from \cite{hajek2016achieving} which works for balanced SBMs, and (ii) SDP$_2$ from \cite[Eq.~(2)]{nguyendifferentially} for imbalanced SBMs with the imbalance ratio input to $1/3$. Figure~\ref{fig:clustering}\,(b) plots the Adjusted RAND Index between the true labels and the predicted labels. Here, the spectral approaches do about as well as SDP$_2$ when the imbalance ratio is known. On the other hand SDP$_1$ is suboptimal compared to the other methods.  Figure~\ref{fig:clustering}\,(c), shows the runtime of the methods. Notably, $k$-means and Algorithm~2 use sparse eigensolvers and scale well, whereas SDP methods can become prohibitively expensive for large graphs.
\end{experiment}



\vspace*{-0.5em}
\section{Conclusion}
\label{sec:conclusion}
\vspace*{-0.5em}

In this work, we address the problem of estimating the latent positions of a GRDPG under \eeldpe{} constraints. If $(\Av, \Xv) \sim \grdpg(\pr, \rho_n; p, q)$ is a GRDPG with latent positions $\X$, we showed that the \eedpe{} synthetic graph $\M_\e(\Av)$ obtained via \eflip{} can be used to construct a privacy-adjusted spectral embedding $\ch{\X}$ that achieves near-minimax optimal rates for sufficiently dense graphs. Additionally, we showed that the class of GRDPGs is closed under \eflip{}, with the latent positions of $\M_\e(\Av)$ related to those of $\Av$ by a simple geometric transformation. This perspective enabled us to tackle more nuanced inferential tasks, such as recovering topological information through persistence diagrams. For future work, it would be interesting to investigate the optimality of recovering the latent positions $\X$ under central $\e$-DP constraints, and to extend the analysis to more general graphical models such as graphons.
\vfill\pagebreak

\endgroup

\clearpage
\vfill\pagebreak
\renewcommand{\bibname}{References}
\renewcommand{\bibsection}{\section*{\bibname}}
\bibliography{refs}
\section*{Checklist}

\begin{enumerate}

\item For all models and algorithms presented, check if you include:
\begin{enumerate}
\item A clear description of the mathematical setting, assumptions, algorithm, and/or model. [\textbf{Yes}]
\item An analysis of the properties and complexity (time, space, sample size) of any algorithm. [\textbf{No}] \textit{The time/space complexity of the algorithms are standard.}
\item (Optional) Anonymized source code, with specification of all dependencies, including external libraries. [\textbf{Yes}]
\end{enumerate}

\item For any theoretical claim, check if you include:
\begin{enumerate}
\item Statements of the full set of assumptions of all theoretical results. [\textbf{Yes}]
\item Complete proofs of all theoretical results. [\textbf{Yes}]
\item Clear explanations of any assumptions. [\textbf{Yes}]
\end{enumerate}

\item For all figures and tables that present empirical results, check if you include:
\begin{enumerate}
\item The code, data, and instructions needed to reproduce the main experimental results (either in the supplemental material or as a URL). [\textbf{Yes}]
\smallskip
\begin{center}
        \textcolor{magenta}{\small \url{https://github.com/sidv23/grdpg-ldp}}
\end{center}
\bigskip
\item All the training details (e.g., data splits, hyperparameters, how they were chosen). [\textbf{Yes}]
        \item A clear definition of the specific measure or statistics and error bars (e.g., with respect to the random seed after running experiments multiple times). [\textbf{Yes}]
        \item A description of the computing infrastructure used. (e.g., type of GPUs, internal cluster, or cloud provider). [\textbf{No}]
\end{enumerate}

\item If you are using existing assets (e.g., code, data, models) or curating/releasing new assets, check if you include:
\begin{enumerate}
\item Citations of the creator If your work uses existing assets. [\textbf{Not Applicable}]
\item The license information of the assets, if applicable. [\textbf{Not Applicable}]
\item New assets either in the supplemental material or as a URL, if applicable. [\textbf{Not Applicable}]
\item Information about consent from data providers/curators. [\textbf{Not Applicable}]
\item Discussion of sensible content if applicable, e.g., personally identifiable information or offensive content. [\textbf{Not Applicable}]
\end{enumerate}


\end{enumerate}
\clearpage
\newpage

\onecolumn
\appendix
\allowbreak
\allowdisplaybreaks


\begingroup
\thispagestyle{plain}
\vspace*{1em}
\aistatstitle{
  {\scshape Appendix}\\
  {\large Signal Recovery from Random Dot-Product Graphs Under Differential Privacy}
}
\endgroup
\vspace*{-5em}


\begingroup
\vspace*{-1.2em}
\renewcommand{\contentsname}{}
\setcounter{tocdepth}{1}
\tableofcontents
\bigskip

A summary of notations is provided in \cref{tab:notation}. Throughout, we use $a_n \lesssim b_n$ and $a_n = O(b_n)$ to denote $a_n \le Cb_n$ for some constant $C > 0$ which may change from line to line but \textit{does not} depend on $\eps$ or $n$. We use $\op(b_n)$ and $\littleop(b_n)$ to denote the usual Mann-Wald asymptotic order for random variables, i.e., $X_n = \op(a_n)$ if $\pr(\abs{X_n/a_n} > C_r) \le n^{-r}$, and $X_n = \littleop(a_n)$ if $\limsup_n \pr(|X_n/a_n| > C) = 0$ for all $C > 0$.

\begin{table}[h!]
    \centering
    \small
    \caption{Notations}\label{tab:notation}
    \vspace*{-10pt}
    \begin{tblr}{width=\linewidth, colspec={X[-1,l] X[3,l]}}
        \toprule[2pt]

        $\o(d),\I[d]$ & The group of $d \times d$ orthogonal matrices, and the identity matrix\\
        $\o(p, q), \I[p,q]$ & The group of $d \times d$ indefinite orthogonal matrices, and the indefinite identity matrix\\
        $\bbB(n)$ & The set of binary, symmetric $n \times n$ matrices $\subset \qty{0,1}^{n\times n}$\\
        $\bbU(n, d)$ & The set (Stiefel manifold) of $n\times d$ matrices with orthonormal columns satisfying $\Uv\tr\Uv = \I[d]$\\
        
        \hline[dashed]
        
        $\L(\Av)$ & The diagonal matrix of eigenvalues of a square matrix $\Av$\\
        $\opnorm{\X}$ & The operator/spectral norm of matrix $\X$ given by $\sqrt{\eigen[\max](\X{}\tr\X)}$\\
        $\ttinf{\X}$ & The two-to-infinity norm of matrix $\X \in \R^{n \times d}$ given by $\max_{1 \le i \le n}\norm{\Xv_{i, *}}_2$\\
        $\norm{\uv}$ & The $\ell_2$-norm of a vector $\uv \in \Rd$\\
        
        \hline[dashed]
        
        $(\Av, \X) \sim \grdpg\qty(\pr, \rho_n; p, q)$ & GRDPG with $(p,q)$-admissible measure $\pr$, sparsity $\rho_n\le 1$ and latent positions $\X \in \Rnd$\\ 
        $\P \equiv \P_\X$ & The expected adjacency matrix of $\Av$, i.e., $\P\equiv\P_\X = \X\I[p, q]\X\tr$\\
        $\D_\X$ & The \textit{empirical} second moment matrix of $\X$, i.e., $\D_\X  = $ for $\xi \sim \pr$\\
        $\D_\xi,\D_{\pr}$ & The \textit{population} second moment matrix of $\pr$, i.e., $\D_\xi \equiv \D_{\pr} = \E(\xi \xi\tr)$ for $\xi \sim \pr$\\
        $\hX = \spec(\Av; d)$ & The adjacency spectral embedding of $\Av$ in $\Rd$; see \cref{eq:xhat-spec}\\
        $\O_\X$ & The $\o(d)$ matrix aligning $\hX$ to $\U_\P\abs{\L_\P}^{1/2}$; see \cref{fact:Ox}\\
        $\Qx$ & The $\o(p, q)$ matrix aligning $\X$ to $\U_\P\abs{\L_\P}^{1/2}$; see \cref{fact:Qx}\\
        $\Qz$ & The population analogue of $\Qx$; see \cref{fact:Qz}\\
        $\oz$ & The matrix aligning $\Qx$ to $\Qz$; see \cref{fact:Oz}\\
        
        \hline[dashed]

        $\pi(\e), \s(\e), \tau(\e)$ & Parameters of \eflip{}; see \cref{eq:se-te}\\
        $\M_\e(\Av)$ & The \eedpe{} synthetic copy of $\Av$ under \eflip{}\\
        $\ch{\Av}, \ch{\X}, \ch{\rho}_n$ & {The privacy-adjusted adjacency matrix, spectral embedding and estimated sparsity 
        }\\

        \hline[dashed]
        
        $\dgm(\X)$ & Persistence diagram associated with the rows of $\X$\\
        $\w(\sfD_1, \sfD_2)$ & Bottleneck distance between two persistence diagrams $\sfD_1, \sfD_2$\\
        $d_H(\bbX, \bbY)$ & Hausdorff distance between two compact sets $\bbX$ and $\bbY$\\
        \bottomrule[2pt]        
    \end{tblr}
\end{table}

\endgroup


\section{Spectral alignment matrices}
\label{sec:alignment}

There are several spectral alignment matrices which appear in the proofs of the main results. This section provides a brief overview of these matrices and their properties. See, also, \cite[Tables~1~\&~2]{agterberg2020nonparametric} and \cite[Section~3]{agterberg2020two} for a comprehensive overview.

Let $\xi \sim \pr$ be a random vector in $\Rd$ with distribution $\pr$ and let $\X \in \Rnd$ whose rows $X_1, X_2, \cdots, X_n$ are \iid{} copies of $\xi$. Following the setup in \cref{sec:results}:

\begin{enumerate}[label=(\roman*)]
    \item Let $\D_\xi := \E(\xi\xi\tr) \in \Rdd$ be the second-moment matrix associated with $\xi$ and let 
    $$
    \D_\X \defeq \frac 1n \X\tr\X \in \Rdd
    $$
    be the \textit{empirical} second-moment matrix associated with $\X$. 
    \item For the edge-probability matrix $\P = \rho_n\X \I[p,q] \X\tr \in [0, 1]^{n\times n}$ under sparsity $\rho_n$, let $\P \defeq \U_\P \L_\P \U_\P\tr$ be its spectral decomposition where $\L_\P$ contains the $d$ eigenvalues of $\P$ such that $\textup{sgn}(\L_\P) = \I[p,q]$ and $U_\P \in \bbU(n, d)$. Define 
    \begin{align}
        \t{\X} \defeq \U_\P \abs{\L_\P}^{1/2} \in \Rnd\label{eq:x-tilde-def}
    \end{align}
    as a surrogate for $\rho_n^{1/2}\X$.
    \item For the random graph $\Av \in \bbB(n)$ such that $\pr(\Av_{ij} = 1) = \P_{ij}$, let 
    $$
    \Av = \U \L \U\tr + \U_\perp \L_\perp \U_\perp
    $$
    be its spectral decomposition of $\Av$ where $\L \equiv \L_\Av$ contains the top-$d$ eigenvalues of $\Av$ \textit{by magnitude} and $\U \in \bbU(n, d)$ contains the corresponding eigenvectors. From \cref{def:ase}, 
    \begin{align}
        \hX \defeq \U \abs{\L}^{1/2} \in \Rnd \label{eq:x-hat-def}
    \end{align}
    is the \textit{adjacency spectral embedding} of $\Av$.
\end{enumerate}

\begin{remark}
    Note that the matrices $\hX$ and $\t{\X}$ depend on the sparsity via rescaling by a factor of $\rho_n^{1/2}$ whereas the matrix $\X$ does not.
\end{remark}

Comparing the expressions in \cref{eq:x-tilde-def} and \cref{eq:x-hat-def} and by noting that $\E(\Av) = \P$, one would expect that $\hX$ is close to $\t{\X}$. However, since $\U_\X, \U \in \bbU(n, d)$ arise from their respective spectral decompositions, they are unique only up to orthogonal transformations. The matrix $\O_\X$ is the orthogonal matrix that aligns $\hX$ to $\t{\X}$.

\begin{fact}\label{fact:Ox}
    $\O_\X \in \o(d)$ is the matrix which solves the orthogonal Procrustes problem:
    \begin{align}
        \O_\X \defeq \argmin_{\O \in \o(d)} \fnorm\big{\hX - \t{\X}\O}^2.\label{eq:w-def}
    \end{align}
\end{fact}

The positions $\t{\X} \in \Rnd$ act as a surrogate for the rescaled latent positions $\rho_n^{1/2}\X$. Since $\t{\X}$ arises from $\P$---which admits an indefinite spectral decomposition, $\t{\X}$ and $\rho_n^{1/2}\X$ are related by an $\o(p,q)$ transformation.

\begin{fact}[{\cite[p.~1457]{rubin2022statistical}, \cite[Section~4]{agterberg2020two} and \cite[Section~3.1]{agterberg2020nonparametric}}]
    \label{fact:Qx}
    The matrix $\Qx \in \o(p,q)$ aligns $\t{\X}$ to $\rho_n^{1/2}\X$, i.e.,
    $$
    \rho_n^{1/2}\X\Qx\inv = \t{\X} = \U_\P \abs{\L_\P}^{1/2}.
    $$
    From \cite[Eq.~(1)]{agterberg2020two}, the expression for $\Qx$ is given by
    \begin{align}
    \Qx \defeq \qty(\frac{1}{n\rho_n}\abs{\L_\P})^{-1/2} \V\tr \D_\X\hpow,\label{eq:Qx}
    \end{align}
    where $\V \in \o(d)$ are the eigenvectors in the spectral decomposition $\D_\X\hpow \I[p,q] \D_\X\hpow = \V \L_2 \V\tr$.
\end{fact}

The expression in \cref{eq:Qx} is slightly different from \cite[Eq.~(1)]{agterberg2020two} since the authors don't consider the rescaling by the sparsity factor $\rho_n^{1/2}$. The stated expression holds by noting the following relationship between the eigenvalues of $\P$ and $\D_\X$.

\begin{fact}\label{fact:spectrum}
The matrices $\frac{1}{n\rho_n}\P$, $\D_\X \I[p,q]$ and $\D_\X^{1/2} \I[p,q] \D_\X^{1/2}$ have the same eigenvalues, i.e.,
$$
\frac{1}{n\rho_n}\L_\P = \L\qty(\rho_n\D_\X \I[p,q]) = \L\qty\Big((\rho_n\D_\X)^{1/2} \I[p,q] (\rho_n\D_\X)^{1/2}).
$$
\end{fact}
\begin{proof}
    The proof is a simple consequence of the property that the eigenvalues associated with the product of compatible matrices is invariant to cyclic permutations (up to a collection of repeated zero eigenvalues). Specifically, if $\mu$ is a non-zero eigenvalue of the matrix product $\Bv\Cv\Dv$ with eigenvector $\uv$, then $\Dv\uv$ is an eigenvector of $\Dv\Bv\Cv$ with eigenvalue $\mu$, i.e., 
    $\Dv\Bv\Cv (\Dv\uv) =  \Dv (\Bv\Cv\Dv \uv) = \Dv \mu \uv = \mu (\Dv \uv).$  
    Therefore, it follows that $\L(\frac{1}{n\rho_n}\P) = \L(\frac{1}{n}\X\I[p,q]\X\tr) = \L(\I[p,q] \frac{1}{n}\X\tr\X) = \L(\I[p,q] \D_\X) = \L(\D_\X\hpow \I[p,q] \D_\X\hpow)$.
\end{proof}

The matrix $\Qx \in \o(p, q)$ from \cref{fact:Qx} has a population analogue, $\Qz$, which is characterized as follows. 

\begin{fact}[{\cite[Lemma~2]{agterberg2020nonparametric}}]\label{fact:Qz}
    For $\D_\xi = \E(\xi\xi\tr)$, consider the spectral decomposition of $\D_\xi\hpow \I[p,q] \D_\xi\hpow$, 
    $$
    \D_\xi\hpow \I[p,q] \D_\xi\hpow = \W \L_\xi \W\tr\qq{where} \W \in \o(d).
    $$
    Then, the population analogue of $\Qx$ is the matrix $\Q_\xi \in \o(p,q) \cap \o(d)$ given by
    \begin{align}
        \Q_\xi \defeq \abs{\L_{\xi}}\pow \W\tr \D_\xi\hpow.\label{eq:Qz}
    \end{align}
\end{fact}

Since $\Qx$ and $\Qz$ are determined by the matrices $\V, \W \in \o(d)$ arising from spectral decompositions, similar to \cref{fact:Ox} it follows that $\Qx$ and $\Qz$ are also unique only up to orthogonal transformations. The matrix $\O_\xi$ aligns $\Qx$ to $\Qz$.

\begin{fact}\label{fact:Oz}
    The matrix $\O_\xi \in \o(d)$ is the matrix which aligns $\Qx$ to $\Qz$ and is given by
    \begin{align}
        \O_\xi \defeq \argmin_{\O\in \o(d)}\fnorm{\qx\inv\O - \qz\inv}^2
    \end{align}
\end{fact}

Finally, the following lemma shows that the matrices $\Qx$ and $\Qz$ are invariant to scale transformations. 

\begin{lemma}
    For $\qx, \qz$ as defined in \cref{eq:Qx} and \cref{eq:Qz}, respectively, $\Q_{t\X} = \Qx$ and $\Q_{t\xiv} = \qz$ for all $t > 0$.
    \label{lemma:qx-scale}
\end{lemma}

\begin{proof}
    For $t > 0$ and $\Y := t\X$,
    $$
    \D_\Y = \tfrac{1}{n}\Y \tr \Y = \tfrac{1}{n} t^2 \X\tr\X = t^2 \D_\X.
    $$ 
    Similarly, we also have $\P_\Y = t^2 \P_\X$. Therefore, $\L_{\P_\Y} = t^2\L_{\P_\X}$, and, for the spectral decomposition in \cref{fact:Qx}, $\V_{\D_\Y} = \V_{\D_\X}$. Plugging these into the expression for $\Q_{\Y}$, we get
    \eq{
        \Q_{\Y} = \qty(\tfrac{1}{n\rho_n}\abs{\L_{\P_\Y}})\pow \V_{\D_\Y}\tr \D_\Y\hpow = \qty(\tfrac{t^2}{n\rho_n}\abs{\L_{\P_\X}})\pow \V_{\D_\X}\tr (t^2\D_\X)\hpow = \Qx.\nn
    }
    A similar argument also shows that $\Q_{t\xi} = \qz$.
\end{proof}



\section{Auxiliary Results}
\label{sec:auxiliary}

In this section, we collect some results which are used in the proofs presented in \cref{sec:proofs}.

\begingroup
\renewcommand{\D}{\boldsymbol{\Delta}_{\qr}}
\renewcommand{\Do}{\boldsymbol{\Delta}_{\pr}}
\renewcommand{\P}{\mathbf{P}_{\Y}}
\renewcommand{\Po}{\mathbf{P}_{\X}}

\subsection{Properties of \textbf{P} and $\Delta$}

For $\e > 0$ and $\X  \sim_{\iid{}} \pr$, let the map $\varphi_\e: \R^d \to \R^{d+1}$ be the map in \cref{prop:closure} given by $\varphi_\e(\xv) = \te \oplus \se\xv$. Define
\begin{align}
    \Y \defeq \varphi_\e(\X) \in \R^{n \times (d+1)} \qc{} \qr \defeq (\varphi_e)\push\pr \qq{and} \etav \sim \qr.
\end{align}

We establish some properties of $\P$ and $\D$ in relation to $\Po$ and $\Do$, which are used extensively in \cref{sec:proofs}. The following result characterizes the spectral properties of $\D$. 


\begin{lemma}
    For $\e > 0$, the following properties hold for $\Do$ and $\D$.
    \begin{enumerate}
        \item $\Do$ and $\D$ are both positive definite. 
        \item If $\eigen[1](\Do) \ge \eigen[2](\Do) \ge \dots \ge \eigen[d](\Do)$ are the eigenvalues of $\Do$, and $\eigen[1](\D) \ge \eigen[2](\D) \ge \dots \ge \eigen[d+1](\D)$ are the eigenvalues of $\D$, then 
        \eq{
            \eigen[1](\D) \ge \se[2] \eigen[1](\Do) \ge \dots \ge \se^2 \eigen[d](\Do) \ge \eigen[d+1](\D). 
        }
        \item For sufficiently small $\e>0$, there exists $C_1 > 0$ such that
        \eq{
            \eigen[1](\D) \le \te^2 + C_1\se\te.\nn
        }
        \item For sufficiently small $\e>0$, there exists $C_2 > 0$ such that 
        \eq{
            \eigen[d+1](\D) \ge C_2 \se[2].\nn
        }
    \end{enumerate}
    \label{lemma:D-matrix}
\end{lemma}

\begin{proof} For notational simplicity, throughout the proof we take $\s = \se$ and $\tau = \te$. 

    \textit{Part 1.} For $\xiv \sim \pr$, $\xiv\xiv\tr$ is positive definite a.e.-$\pr$. To see this, note that for any $\xv \in \R^d$, 
    \eq{
        \xv\tr\pa{\xiv\xiv\tr}\xv = \norm{\xv\tr\xiv} \ge 0 \quad \text{a.e.}-\pr. \nn
    }
    It follows that $\E\qty(\xv\tr\pa{\xiv\xiv\tr}\xv) = \xv\tr\E(\xiv\xiv\tr)\xv = \xv\tr\Do\xv \ge 0$. Since $\Sigmav = \textup{Cov}(\xiv)$ is assumed to be full-rank, it follows that $\xv\tr\Do\xv > 0$ for all $\xv \in \R^d$. Therefore, $\Do$ is positive definite. Let $\etav = \phi(\xiv) = \te \oplus \se \xiv$. Because $\D = \E\pa{ \etav\etav\tr}$, from a similar argument it follows that $\xv\tr\D\xv \ge 0$ for all $\xv \in \R^{d+1}$. It remains to show that $\xv\tr\D\xv > 0$. To this end, note that $\D$ is the block matrix given by
    \eq{\label{eq:D-def}
        \D = \begin{bmatrix}
            \tau^2 & \s\tau \E(\xiv)\tr\\ 
            \s\tau\E(\xiv) & \s^2 \Do\\
        \end{bmatrix}.
    }
    It is easy to verify that the determinant of $\D$ is given by 
    \eq{\label{eq:D-det}
        \det(\D) = \s^{2d}\tau^2 \det(\Do) \pa{ 1 - \E(\xiv)\tr\Do\inv\E(\xiv) }. 
    }
    If we can show that $\det(\D) > 0$, or, equivalently that $\E(\xiv)\tr\Do\inv\E(\xiv) < 1$, then by Sylvester's criterion \citep[Theorem~7.2.5]{horn2012matrix}, it will follow that $\D$ is positive definite. With this in mind, let $\zv = \E(\xiv)$. Note that 
    $$
    \Do = \E(\xiv\xiv\tr) = \textup{Cov}(\xiv) + \E(\xiv)\E(\xiv)\tr = \sig + \zv\zv\tr,
    $$ 
    where $\sig$ is positive definite. Then, using the Sherman-Morrison-Woodbury formula \citep[Section~0.7.2]{horn2012matrix},
    \eq{
        \zv\tr\Do\inv\zv &= \zv\tr\qty(\sig\inv - \f{\sig\inv \zv\zv\tr \sig\inv}{1 + \zv\tr\sig\inv\zv})\zv\nn\\
        &= \zv\tr\sig\inv\zv - \f{1 + \zv\sig\inv\zv}{(\zv\tr\sig\inv\zv)^2}\nn\\
        &= \f{\zv\tr\sig\inv\zv}{1 + \zv\tr\sig\inv\zv} < 1. 
    }
This implies that $1 - \zv\tr\Do\inv\zv > 0$, and, therefore, $\D$ is full-rank and positive definite. 

\textit{Part 2.} Using Cauchy's interlacing theorem \citet[Theorem~4.3.17]{horn2012matrix} for the block-matrix representation of $\D$ in \eref{eq:D-def}, we obtain
\eq{
    \eigen[1](\D) \ge \eigen[1](\s^2 \Do) \ge \eigen[2](\D) \ge \eigen[2](\s^2 \Do) \ge \dots \ge \eigen[d](\s^2\Do) \ge \eigen[d+1](\D).\nn
}
By noting that $\eigen[i](\s^2 \Do) = \s^2 \eigen[i](\Do)$, the result for \textit{(2)} follows.

\textit{Part 3.} The Ger\v{s}gorin disk theorem \citep[Theorem~6.1.1]{horn2012matrix} for $\D$ asserts that for all $1 \le k \le d+1$, the collection of eigenvalues of $\D$ satisfy
\eq{
    \eigen[k](\D) \in \bigcup_{1 \le i \le d+1}\qty\Big[ (\D)_{ii} - \sum_{j\neq i}(\D)_{ij}, (\D)_{ii} + \sum_{j\neq i}(\D)_{ij} ].\nn
}
This implies that $\eigen[\max](\D) = \eigen[1](\D) \le \tau^2 + \s\tau\sum_{1 \le i \le d}\E(\xi)$. Taking $C_1 = \onev[d] \tr \E(\xiv)$, the result follows.

\textit{Part 4.} Taking $\omega = 1 - \E(\xiv)\tr\Do\inv\E(\xiv) > 0$ and using the fact that $\det(\D) = \prod_{1\le i \le d+1}\eigen[i](\D)$, we obtain 
\eq{
    \eigen[d+1](\D) &= \f{\det(\D)}{\eigen[1](\D) \times \prod_{j=2}^{d}\eigen[j](\D)}\nn\\
    &\en{i} \f{\s^{2d}\tau^2 \det(\Do) \pa{ 1 - \E(\xiv)\tr\Do\inv\E(\xiv) }}{\eigen[1](\D) \times \prod_{j=2}^{d}\eigen[j](\D)}\nn\\
    &= \s^{2d}\tau^2\omega \f{\prod_{j=1}^d\eigen[j](\Do)}{\eigen[1](\D) \times \prod_{j=2}^d\eigen[j](\D)}\nn\\
    &=  \f{\eigen[d](\Do)\tau^2\omega}{\eigen[1](\D)}\times \s^2 \times \left({\prod_{j=1}^{d-1}\s^{2} \eigen[j](\Do)}  \middle/ {\prod_{j=2}^d\eigen[j](\D)} \right)\nn\\
    &\en{ii}[\ge]  \pa{\f{\eigen[d](\Do)  \tau^2\omega}{\tau^2 + \s \tau C_1}} \s^2\nn\\
    &\en{iii}[\ge] \s^2 C_2,\nn
}
where (i) follows from the definition of $\det(\D)$ from \eref{eq:D-def}, (ii) uses the interlacing property of the eigenvalues in part \textit{2}, i.e. $\eigen[j](\D) \le \s^2 \eigen[j-1](\Do)$ for $j \in \pb{2 \dots d}$, and (iii) follows from taking $C_2 = \eigen[d](\Do) \tau^2 \omega / ({\tau^2 + \s \tau C_1}) > 0$. 

\end{proof}


The next lemma characterizes the spectral properties of $\P$. 

\begin{lemma}
    For $\e > 0$, the following properties hold for $\Po$ and $\P$.
    \begin{enumerate}
        \item Up to a collection of repeated zero eigenvalues, 
        $$
            \L\pa{ \P } = \L\qty\big({\yty \I[p+1, q]}) = \L\pa{ \pa{\yty}\pow \I[p+1,q] \pa{\yty}\pow}.
        $$
        \item With probability greater than $1 - 2(d+1)/n$, 
        $$\norm{\yty\I[p+1,q] - n\D\I[p+1,q]} \le C\sqrt{n\log n}.$$
        \item $\abs{ \eigen[d+1](\P) } = \Omega_{\pr}\qty\Big({ n\se[2] })$.
    \end{enumerate}
    \label{lemma:P-matrix}
\end{lemma}

\begin{proof}
    As before, for notational simplicity throughout the proof we take $\s = \se$ and $\tau = \te$. 
    
    \textit{Part 1.} The first claim follows from the same argument as in \cref{fact:spectrum}.
    
    \textit{Part 2.} First, we note that $\yty$ can be written as the sum of iid random matrices
    \eq{
        \yty = \sum_{i=1}^n \Yv_i \tr \Yv_i \in \R^{d+1 \times d+1},\nn
    }
    with $\E(\Yv_i\tr\Yv_i) = \D$. Since $\pr$ has compact (and, therefore, bounded) support, each $\Xv_i \sim \pr$ has $\norm{\Xv_i} \le \diam(\bX) = L < \infty$ a.e.-$\pr$. Consequently, $\norm{\Yv_i} = \sqrt{\Yv_i\tr\Yv_i} = \norm{\tau \oplus \s \Xv_i} \le \sqrt{\tau^2 + L^2} = M < \infty$ a.e. For $t>0$, we can use the matrix Bernstein inequality \citep[Theorem~6.1.1]{tropp2015introduction} to obtain the tail bound
    \eq{
        \pr\qty\Big( \norm{\yty - n\D} > t ) \le 2d \exp{ \f{-t^2}{2 v(\yty) + \f{2}{3}
        Mt} },
        \label{eq:bernstein}
    }
    where $v(\yty) = n \norm{\E(\Yv_i\tr\Yv_i\Yv_i\tr\Yv_i)} \le n \max\pb{ \tau^2 + \s\tau C_1, M^2 } = n k$ for $k = \max\pb{ \tau^2 + \s\tau C_1, M^2 }$. This follows from the fact that $E(\Yv_i\tr\Yv_i) = \D$ is bounded in spectral norm from Lemma~\ref{lemma:D-matrix}(\textit{3}), and $E(\Yv_i\Yv_i\tr) \le \norm{\Yv_i}^2 \le M^2$. For $\delta > 0$, \eref{eq:bernstein} is equivalent to 
    \eq{
        \pr\qty\Bigg( \norm{\yty - n\D} \le \sqrt{{2 k \delta n}} + \f{2M}{3} \delta) \ge 1 - 2d \ e^{-\delta}.\nn
    }
    Taking $\delta = \log n$ and $C$ sufficiently large, it follows that
    \eq{
        \pr\qty\Big(\norm{\yty - n\D} \le C \sqrt{n \log n}) \ge 1 - \f{2(d+1)}{n}.\nn
    }
    By noting that 
    $$
    \norm{\yty\I[p+1, q] - n\D\I[p,q]} \le \norm{\yty - n\D} \cdot \norm{\I[p+1,q]} = \norm{\yty - n\D},
    $$
    we obtain the desired result, i.e. with probability greater than $1- 2(d+1) n\inv$
    \eq{\label{eq:bernstein2}
        \norm{\yty\I[p+1, q] - n\D\I[p,q]} \le C\sqrt{n\log n}. 
    }

    \textit{Part 3.} We begin by noting that
    \eq{
        \qty\Big| \abs\big{\eigen[i](\yty\I[p+1, q])} - \abs\big{\eigen[i](n\D\I[p+1, q])}| 
        &\en{i}[\le] \qty| {\eigen[i](\yty\I[p+1, q])} - {\eigen[i](n\D\I[p+1, q])} |\nn\\ 
        &\en{ii}{\le} \norm{\yty\I[p+1, q] - n\D\I[p+1, q]},\label{eq:bernstein3}
    }
    where (i) follows from the reverse triangle inequality, and (ii) is a consequence of Weyl's perturbation theorem \citep[Theorem~4.3.1]{horn2012matrix}. Therefore, the eigenvalues of $\yty\I[p+1, q]$ can be controlled by the eigenvalues of $n\D\I[p+1, q]$. We now obtain a lower bound on $\eigen[i](n\D\I[p+1, q]) = n\eigen[i](\D\I[p+1, q])$. 

    From \textit{Part (1)} we have that $\eigen[i](\D\I[p+1, q]) = \eigen[i](\I[p+1, q]\D)$. Furthermore, the singular-values of $\I[p+1, q]\D$ are given by 
    $$
    \singular[i](\I[p+1, q]\D) =  \sqrt{ \eigen[i]\qty\Big( \qty(\I[p+1, q]\D)\tr\I[p+1, q]\D ) } = \sqrt{ \eigen[i](\D^2) } = \abs{\eigen[i](\D)}.
    $$ 
    From \citet[Theorem~5.6.9]{horn2012matrix}, it follows that 
    \eq{
        \abs{\eigen[\min](\D\I[p+1, q])} = \abs{\eigen[\min](\I[p+1, q]\D)} \ge \singular[\min](\I[p+1, q]\D) = \abs{\eigen[\min](\D)} \ge \s^2 C_2,\nn
    }
    
    where the last inequality follows from Lemma~\ref{lemma:D-matrix}. Therefore, we obtain the lower bound 
    \eq{\label{eq:bernstein4}
        \abs{\eigen[\min](n\D\I[p+1, q])} \ge n\s^2 C_2.
    }

    Similarly, we can also find an upper bound for $\abs{\eigen[\max](\D\I[p+1, q])}$ using \citet[Theorem~5.6.9]{horn2012matrix} by observing that
    \eq{\label{eq:bernstein5}
    \abs{\eigen[\max](\D\I[p+1, q])} \le \singular[\max](\D\I[p+1, q]) \le \abs{\eigen[\max](\D)} = \tau^2 + \s\tau C_1.\nn
    }
    This implies that $\abs{\eigen[\max](n\D\I[p+1, q])} = \Theta(n)$. Moreover, combining \eref{eq:bernstein3} with \eref{eq:bernstein2} we obtain
    \eq{
        \pr\qty\Big( \abs{\eigen[d+1](\yty\I[p+1, q])} >  \abs{\eigen[d+1](n\D\I[p+1, q])} - C\sqrt{n\log n}  ) \ge 1 - \f{2(d+1)}{n},
    }
    which, when combined with \eref{eq:bernstein4} yields
    \eq{
        \pr\qty\Big( \abs{\eigen[d+1](\yty\I[p+1, q])} >  n\s^2 C_2 - C\sqrt{n\log n}  ) \ge 1 - \f{2(d+1)}{n}.\nn
    }
    When $\s > \f{2C}{C_2}\sqrt{{\log n}/{n}}$, or, alternatively, when $\s = \omega\qty(\sqrt{{\log n}/{n}})$, the lower bound is non-trivial. Therefore, when $\s = \omega\qty(\sqrt{{\log n}/{n}})$ it follows that
    \eq{
        \pr\qty\Bigg( \abs{\eigen[d+1](\yty\I[p+1, q])} >  \f{C_2}{2}n\s^2  ) \ge 1 - \f{2(d+1)}{n}.\nn
    }
    The result follows by noting that this is equivalent to the statement that $\abs{\eigen[d+1](\P)} = \Omega_{\pr}(n\s^2 )$.
\end{proof}


We collect the main findings from Lemma~\ref{lemma:D-matrix} and Lemma~\ref{lemma:P-matrix} in the following corollary. 

\begin{corollary}
    For $\P$, $\D$ and their associated eigenvalues $\eigen[i](\P)$ and $\eigen[i](\D)$:
    \begin{enumerate}
        \item $\eigen[\min](\D) = \Omega(\se^2)$ and $\eigen[\max](\D) = O(1)$. 
        \item $\eigen[\min](\P) = \Omega(n\se^2)$ and $\eigen[\max](\P) = \Theta(\eigen[\max]{(\Po)})$.
    \end{enumerate}
    \label{corollary:spectral-norm}
\end{corollary}


For $\M_\e(\Av) =: \Av_\Y \sim \grdpg(\Y; p+1, q)$, next lemma establishes a tail bound for $\norm{\Av_\Y - \Pv_\Y}$ using a straightforward application of \citet[Theorem~5.2]{lei2015consistency}. 

\begin{lemma}Under the conditions of Theorem~\ref{prop:grdpg-estimation}, there exists a constant $C>0$ such that with probability greater than $1 - 1/n$,
    \eq{
        \norm{\Av_\Y - \Pv_\Yv} \le C\sqrt{ n \gamma },\nn
    }
    where $\gamma = \se[2] \max_{ij}(\Po)_{ij} + \te[2] \le 1$. 
    \label{lemma:lei1}
\end{lemma}

\begin{proof}
    We may directly apply \citet[Theorem~5.2]{lei2015consistency} to obtain the claim provided we verify the following two conditions: (i) $\E\qty\big(\M_\e(\Av)) = \P$, and (ii) $n \max_{i, j}(\P)_{ij} < n \gamma$. Note that the additional requirement that $n \gamma > n\te^2 > c_0 \log n$ for a fixed $c_0 > 0$ is trivially satisfied. To verify the first condition, for $1 \le i < j \le n$ we have that
    \eq{
        \E\qty\bigg(\M_\e(\Av_{ij})) &\en{i} \pr\qty\Big({ \M_\e(\Av_{ij}) = 1})\nn\\
        &\en{ii} \pr\qty\Big({ \M_\e(\Av_{ij}) = 1} \mid \Av_{ij = 1}) \pr\pa{\Av_{ij} = 1} + \pr\qty\Big({ \M_\e(\Av_{ij}) = 1} \mid \Av_{ij = 0}) \pr\pa{\Av_{ij} = 0}\nn\\
        &= \qty\big(1 - \pi(\e)) \Xv_i\tr\I[p,q]\Xv_j + \pi(\e) (1-\Xv_i\tr\I[p,q]\Xv_j)\nn\\
        &= \se[2]\Xv_i\tr\I[p,q]\Xv_j + \te[2] = (\P)_{ij},\nn
    }
    where (i) uses the fact that $\M_\e(\Av_{ij})$ is from a Bernoulli distribution, and (ii) uses the definition of \eflip{}. For the second condition, note that 
    \eq{
        \max_{i,j \in [n]}(\P)_{ij} = \max_{i,j \in [n]}\se[2]\Xv_i\tr\I[p,q]\Xv_j + \te[2] = \se[2]\qty\Big(\max_{i,j \in [n]}\Xv_i\tr\I[p,q]\Xv_j) + \te[2] = \gamma \le 1.\nn
    }
    Therefore, the claim that $\norm{\Av_\Y - \P} = \op\pa{\sqrt{n}}$ follows from \citet[Theorem~5.2]{lei2015consistency}. 
\end{proof}

\endgroup


\subsubsection{Upper bounds for GRDPG Estimation}

In this section, recall some key results from \cite{rubin2022statistical,agterberg2020nonparametric,yan2023minimax} which are used in the proofs of our main results. The alignment matrices appearing here are described in \cref{sec:alignment}.

Recall that for $\X, \Y \in \Rnd$, the $\ell_{2\to\infty}$ metric between $\X$ and $\Y$ which reproduces the non-identifiability of the GRDPG is given by
\begin{align}
    \dttinf(\X, \Y) = \inf_{\O \in \o(p,q) \cap \o(d)} \ttinf\Big{\Y\Qy\inv - \X\Qx\inv\O}.\label{eq:dttinf-2}
\end{align}
where the matrices $\Qx, \Qy$ are the alignment matrices for $\X$ and $\Y$, respectively, as defined in \cref{fact:Qx}. Equivalently,
\begin{align}
    \dttinf(\X, \Y) = \inf_{\O \in \o(p,q) \cap \o(d)} \ttinf\Big{\U_\Y\abs{\L_\Y}^{1/2} - \U_\X\abs{\L_\X}^{1/2}\O}.\label{eq:dttinf-3}
\end{align}

The following result from \cite{agterberg2020nonparametric} is a restatement of the main result from \cite{rubin2022statistical} for estimating the latent positions of GRDPGs under the $\dttinf$ metric.

\begin{proposition}[Theorem~4 of \cite{agterberg2020nonparametric}]\label{prop:grdpg-estimation}
    Given $(p,q)$-admissible $\pr$, let $(\Av, \Xv) \sim \grdpg(\pr; p, q, \rho_n)$, and let $\hX = \spec(\Av; d) \in \Rnd$ be the adjacency spectral embedding of $\Av$. If $n\rho_n = \omega(\log^4{n})$, then with probability $1 - O(n^{-2})$,
    \begin{align}
        \dttinf\qty(\hX(\Av), \rho_n^{1/2}\X) = O\qty(\frac{\log{n}}{\sqrt{n}}).
    \end{align}
    Equivalently, for estimating $\X$, it follows that with probability $1 - O(n^{-2})$,
    \begin{align}
        \dttinf\qty(\rho_n^{-1/2}\hX(\Av), \X) = O\qty(\frac{\log{n}}{\sqrt{n\rho_n}}).
    \end{align}
\end{proposition}

In addition to the above result, we also need the following key property of the matrices $\Qx$ and its population analogue $\Qz$. The expressions for $\Qz, \Qz$ are given in \cref{fact:Qx} and \cref{fact:Oz}, respectively. The following result is an improvement of \cite[Lemma~16]{solanki2019persistent} using the matrix concentration result from \cite[Lemma~2]{agterberg2020nonparametric}.

\begin{lemma}
    For $\D\equiv \D_{\pr}$ and $\X  \sim_{\iid{}} \pr$, let $\Qx \in \o(p, q)$ and $\Qz \in \o(p, q) \cap \o(d)$ be the alignment matrices described in \cref{fact:Qx} and \cref{fact:Qz}, respectively. Then, there exists a block orthogonal matrix $\O \in \o(d) \cap \o(p, q)$ such that
    \eq{
        \norm{\Qx\inv - \Qz\inv\O} = \op\pa{\sqrt{\frac{\log n}{n}}}.\nn
    }
    Moreover, by \cref{lemma:qx-scale}, the result holds true for arbitrary scale transformations $t\X$ and $t\xi$ for $t > 0$.
    \label{lemma:subspace-alignment}
\end{lemma}



\section{Proofs}
\label{sec:proofs}

In this section, we present the proofs for the main results in \cref{sec:results}. We refer the reader to \cref{tab:notation} for a summary of notations used in this section. The proofs also make use of the supporting results in \cref{sec:auxiliary}.

Throughout the proofs, for $\e > 0$ we use $\sigma \defeq \se$ and $\tau \defeq \te$ for brevity. For the map $\varphi \defeq \varphi_\e: \R^{d} \to \R^{d+1}$ given by $\varphi_\e(\xv) = \tau \oplus \s\rho_n\hpow\xv$, we define $\Y = \varphi(\X)$ and $\qr = (\varphi)\push\pr$ for the pushforward measure, and $\etav \sim \qr$ to denote an arbitrary random variable distributed according to $\qr$.


\subsection{Proof of \cref{prop:closure}}
\label{proof:prop:closure}

For $(\Av, \X) \sim \grdpg\pa{\pr, \rho_n; p, q}$, let $\Av^\circ = \M_\e(\Av)$ and $\Av^\bullet \sim \grdpg(\varphi{}\push\pr, 1; p+1, q)$. Since each entry $\Av^\circ_{ij}, \Av^\bullet_{ij} \in \pb{0, 1}$ are independent, it suffices to show that 
\begin{align}
    \Y \sim_{\iid{}} \qr \qq{and} \Av^\circ_{ij} \distas \Av^\bullet_{ij} \qq{for all}  1 \le i < j \le n. 
\end{align}

For the first claim, note that for measurable sets $A_1 \cdots A_n \subset \R^{d+1}$,
\begin{align}
    \qr( Y_1 \in A_1, \cdots, Y_n \in A_n ) &= \pr( X_1 \in \varphi\inv(A_1),  \cdots, X_n \in \varphi\inv(A_n) ) = \prod_{i=1}^n \pr(\varphi\inv(A_i)) = \prod_{i=1}^n\qr(A_i),
\end{align}
where the third equality follows from the fact that $X_i$ are iid and the final equality follows from the definition of the pushforward $\varphi\push\pr$. 

For the second claim, from Definition~\ref{def:eflip}, observe that
\eq{
    \pr\pa{ \Av^\circ_{ij} = 1 \mid \X } 
    &= 
        \pr\qty\Big({ \pb{\F(\Av_{ij}) = 1 - \Av_{ij}} \cap \pb{\Av_{ij} = 0} \mid \X}) 
        +   
        \pr\qty\big({ \pb{\F(\Av_{ij}) = \Av_{ij}} \cap \pb{\Av_{ij} = 1} \mid \X})\nn\\
    &= \pi(\e) \pr\qty\Big({ {\Av_{ij} = 0} \mid \X}) + \qty\big(1 - \pi(\e)) \pr\qty\Big({  {\Av_{ij} = 1} \mid \X})\nn\\[5pt]
    &\en{i} \pi(\e) \qty(1 - \rho_nX_i\tr\I[p,q]X_j) + \qty\big(1 - \pi(\e)) \rho_nX_i\tr\I[p,q]X_j\nn\\[7pt]
    &= \pi(\e) + \qty\big( 1 - 2\pi(\e)) \rho_nX_i\tr\I[p,q]X_j\nn\\[7pt]
    &= \te[2] + \se[2] \rho_n X_i\tr\I[p,q]X_j\nn\\[5pt]
    &\en{ii} 
    \begin{pmatrix}
        \te & \se \rho_n\hpow X_i\tr
    \end{pmatrix}
    \begin{pmatrix}
        1 & 0 \\ 0 & \I[p,q]
    \end{pmatrix}
    \begin{pmatrix}
        \te \\ \se \rho_n\hpow X_j
    \end{pmatrix}\\
    &= Y_i \I[p+1, q] Y_j\\[7pt]
    &= \pr\pa{ \Av^\bullet_{ij} = 1 | \Y },\nn
}
where (i) follows from the fact that $\Av \sim \grdpg\pa{\pr, \rho_n; p, q}$, and (ii) follows from the inner-product in $\R^{d+1}$ w.r.t. the indefinite identity matrix $\I[p+1, q]$ with signature $(p+1, q)$.

\qed


\subsection{Composition of \eflip{}s and \cref{prop:composition}}
\label{sec:composition}

Given $(\Av, \Xv) \sim \grdpg(\pr, \rho_n; p, q)$, let $\M_{\e_2}( \M_{\e_1}(\Av) )$ be edge-DP graph obtained after applying \eflip{} successively with privacy budgets $\e_1, \e_2 > 0$. Let $\tau_i = \tau(\e_i)$, $\s_i = \s(\e_i)$ and $\pi_i = \pi(\e_i)$ for $i=1,2$. 

On the one hand, taking $\Av^{(1)} = \M_{\e_1}(\Av)$ and $\Av^{(2)} = \M_{\e_2}(\Av^{(1)})$, from \cref{prop:closure} it follows that
\begin{align}
    \qty\big(\Av^{(1)}, \varphi_1(\X)) \sim \grdpg\qty\Big(\varphi_1{}\push\pr, 1; p+1, q) \qq{and} \qty\big(\Av^{(2)}, \varphi_2\circ\varphi_1(\X)) \sim \grdpg\qty\Big((\varphi_2\circ\varphi_1)\push\pr, 1; p+2, q),
\end{align}
where $\varphi_1: \Rd \to \R^{d+1}$ and $\varphi_2: \R^{d+1} \to \R^{d+2}$ are given by $\varphi_1(\xv) = \tau_1 \oplus \sigma_1\xv$ and $\varphi_2(\yv) = \tau_2 \oplus \sigma_2\yv$. Specifically, for $\varphi = \varphi_2 \circ \varphi_1$ given by
\begin{align}
    \varphi(\xv) = \qty\Big( \tau_2, \s_2 \tau_1, \s_2 \s_1\xv\tr )\tr \in \R^{d+2},\label{eq:composition-1}
\end{align}
the latent positions of $\Av^{(2)}$ are $\varphi(\X) \in \R^{n \times (d+2)}$. On the other hand, we have
\begin{align}
    \pr\qty\Big( \Av^{(2)}_{ij} = 1 ) 
    &= (1-\pi_2) \cdot \pr\qty\Big( \Av^{(1)}_{ij} = 1 )\\ 
    &\qquad + \pi_2 \cdot \pr\qty\Big( \Av^{(1)}_{ij} = 0 )\\
    &= (1-\pi_2) \cdot \qty\Big{ (1-\pi_1) \cdot \pr(\Av_{ij}=1) + \pi_1\cdot \pr(\Av_{ij}=0) }\\ 
    &\qquad + \pi_2 \cdot \qty\Big{  \pi_1 \cdot \pr(\Av_{ij}=1) + (1-\pi_1)\cdot \pr(\Av_{ij}=0)  }\\
    &= \qty\Big(1 - \pi_1 - \pi_2 + 2\pi_1\pi_2) \cdot \pr(\Av_{ij}=1)\\ 
    &\qquad + \qty\Big(\pi_1 + \pi_2 - 2\pi_1\pi_2) \cdot \pr(\Av_{ij}=0)\\[5pt]
    &= (1-\pi') \cdot \pr(\Av_{ij}=1) + \pi' \cdot \pr(\Av_{ij}=0),
\end{align}
where $\pi' \defeq (\pi_1 + \pi_2 - 2\pi_1\pi_2)$. In other words, using \eref{eq:se-te}, for $\e'$ given by
$$
\e' = \log\qty( \frac{1}{\pi_1 + \pi_2 - 2\pi_1\pi_2} - 1 ),
$$ 
it follows that $\M_{\e'}(\Av) \distas \M_{\e_2} \circ \M_{e_1}(\Av)$, i.e., $\Av^{(2)}$ is equivalently the graph obtained by applying \eflip{} to $\Av$ with privacy budget $\e'$. Using \cref{prop:closure} to this characterization would imply that the latent positions of $\Av$ are given by $\varphi'(\X) \in \R^{d+1}$ where 
\begin{align}
    \varphi'(\xv) = \tau(\eps') \oplus \sigma(\eps')\xv.\label{eq:composition-2}
\end{align}
The following result is a more general version of \cref{prop:composition}, and shows that the two characterizations above are equivalent. In particular, only one extra dimension is needed to faithfully encode the latent positions of the graph after several applications of \eflip{}.

\begin{proposition}\label{prop:multiple-eflips}
    Under the same setting as \cref{prop:closure}, for $\e_1 \cdots \eps_m > 0$, let $\Av^{(m)} = \M_{\e_1} \circ \cdots \circ \M_{\e_m}(\Av)$ be the graph obtained after applying \eflip{} $m$ times. Let $\varphi_{\eps_i}: \R^{m+i-1}\to\R^{m+i}$ be given by $\varphi_{\eps_i}(\yv) = \tau(\eps_i) \oplus \sigma(\eps_i)\yv$ and let
    $$
    \varphi^{(m)} = \varphi_{\eps_1} \circ \varphi_{\eps_2} \circ \cdots \circ \varphi_{\eps_m}: \R^d \to \R^{d+m}.
    $$
    Then, there exists $\Qv \in \o(p+m, q)$ and $a, b > 0$ such that
    \begin{align}
        \Qv\varphi^{(m)}(\xv) = (\zerov[{m-1}], a, b\xv\tr)\tr \quad \forall \xv \in \R^d,
    \end{align}
    and for $\psi: \R^d \to \R^{d+1}$ given by $\psi(\xv) = (a, b\xv\tr)\tr$, it follows that
    $$
    (\Av^{(m)}, \psi(\X)) \sim \grdpg(\psi\push\pr, 1; p+1, q).
    $$
\end{proposition}

\begin{proof}[Proof of \cref{prop:multiple-eflips}]
    Let $\varphi_i = \varphi_{\eps_i}, \tau_i=\tau(\eps_i)$ and $\sigma_i = \sigma(\eps_i)$ for all $i \in [m]$.  The proof follows by induction on $m$. For $m=1$, the result holds trivially from \cref{prop:closure}. By induction hypothesis, suppose the claim holds for $\Av^{(m-1)}$, $m \ge 2$, i.e., there exists $\W \in \o(p+m-1, q)$ and $\alpha, \beta > 0$ such that 
    \begin{align}
        \W\varphi^{(m-1)}(\xv) = \qty(\zerov[{m-2}], \alpha, \beta\xv\tr)\tr \in \R^{d+m-1} \qq{for all} \xv \in \R^d.
    \end{align}
    For $\xv \in \R^{d}$, let $\yv = \varphi^{(m-1)}(\xv) = \qty\big(y_1, y_2, \cdots, y_{m-1}, \gamma \xv\tr)\tr \in \R^{d+m-1}$ where\footnote{A simple calculation shows that $y_1 = \tau_{m-1}$ and $\gamma = \prod_{i=1}^{m-1}\sigma_i$.} $\gamma > 0$. By definition of $\varphi_{m}$, we have 
    $$
    \varphi_{m}(\yv) = \qty(\tau_m, \sigma_m\yv\tr)\tr = \qty(\tau_m, \sigma_m y_1, \sigma_m \yv_{2:m-1}, \sigma_m \gamma \xv\tr)\tr.
    $$
    Let $R \in \o(2)$ be the rotation matrix such that $R(\tau_m, \sigma_m y_1)\tr = (0, \delta y_1)\tr$, i.e., 
    $$
    R = \begin{pmatrix}
        \cos\theta & -\sin\theta\\
        \sin\theta & \cos\theta
    \end{pmatrix}
    \qq{for} \theta = \arctan\qty(\frac{\tau_m}{\sigma_m y_1}),
    $$
    and let $\Q_1$ be the block-orthogonal matrix given by
    $$
    \Q_1 = \begin{pmatrix}
        R & \mathsf{O}_{2 \times (d+m-2)}\\
        \mathsf{O}_{(d+m-2) \times 2} & \I[d+m-2]
    \end{pmatrix} \in \R^{(d+m) \times (d+m)}.
    $$
    By construction, $\Q_1\varphi_{m}(\yv) = (0, \delta y_1, \sigma_m\yv_{2:m-1}, \sigma_m\gamma\xv\tr)\tr$. Similarly, let $\Q_2$ be the matrix given by
    $$
    \Q_2 = \begin{pmatrix}
        1 & \zerov[{(p+m-1)}]\tr\\
        \zerov[{(p+m-1)}] & \Wv
    \end{pmatrix} \in \R^{(p+m) \times (p+m)}.
    $$
    From the induction hypothesis, it follows that $\Q_2\Q_1 \varphi_{m}(\yv) = (0, \zerov[{m-2}], \sigma_m \alpha, \sigma_m \beta \xv\tr)\tr =: (\zerov[m-1], a, b\xv\tr)\tr$. Therefore, for $\Q \defeq \Q_2\Q_1$ and $\varphi^{(m)} = \varphi_{m} \circ \varphi^{(m-1)}$,
    $$
    \Q\varphi^{(m)}(\xv) = (\zerov[m-1], a, b\xv\tr)\tr \in \R^{d+m}.
    $$
    It remains to show that $\Q = \Q_2\Q_1 \in \o(p+m, q)$. This follows by noting that $\Q_2 \in \o(p+m, q)$ since
    \begin{align}
        \Q_2\tr \I[p+m, q] \Q_2 = \begin{pmatrix}
            1 & \zerov[{(p+m-1)}]\tr\\
            \zerov[{(p+m-1)}] & \Wv\tr\I[p+m-1, q]\Wv
        \end{pmatrix}
        = \begin{pmatrix}
            1 & \zerov[{(p+m-1)}]\tr\\
            \zerov[{(p+m-1)}] & \I[p+m-1, q]
        \end{pmatrix}
        = \I[p+m, q],
    \end{align}
    and $\Q_1 \in \o(p+m, q)$ since
    \begin{align}
        \Q_1\tr \I[p+m, q] \Q_1 
        = \begin{pmatrix}
            R\tr \I[2, 0] R & \mathsf{O}_{2 \times (d+m-2)}\\
            \mathsf{O}_{(d+m-2) \times 2} & \I[p+m-2, q]
        \end{pmatrix} 
        = \I[p+m, q]
    \end{align}
    Therefore, it follows that
    $$
    (\Q_2\Q_1)\tr \I[p+m, q] (\Q_2\Q_1)\tr = \Q_1\tr \qty(\Q_2\tr \I[p+m, q] \Q_2) \Q_1 = \Q_1\tr \I[p+m, q] \Q_1 = \I[p+m, q].
    $$
\end{proof}


\begingroup
\newcommand{\Dpr}{\D_{\pr}}
\newcommand{\Dqr}{\D_{\qr}}
\renewcommand{\tv}{\mathsf{D}_{\mathsf{TV}}}
\newcommand{\kl}{\mathsf{D}_{\mathsf{KL}}}
\newcommand{\chisq}{\mathsf{D}_{\chi^2}}
\newcommand{\Xo}{\X_0}
\newcommand{\Uo}{\U_0}
\newcommand{\trr}{{}^{{\scalebox{0.75}{$\mathbf{\mathsf{T}}$}}}}
\newcommand{\T}{\mathcal{A}}

\subsection{Proof of \cref{thm:minimax}}
\label{proof:thm:minimax}

The outline of the proof is as follows:
\begin{enumerate}[label=\Circled{\arabic*}, itemsep=0.0em]
    \item \textbf{Finite Reduction:} We consider a finite set of latent positions $\calX_0, \calX_1 \subset \calX$ which are $\eta$-separated in the $\dttinf$ metric to reduce the minimax risk over $\calX$ to the minimax risk over a finite subset. 
    \item \textbf{Le Cam's Lemma:} Next, we apply Le Cam's lemma to lower bound the minimax risk over $\calX_0 \cup \calX_1$ in terms of the total variation metric between mixture distributions of the private outputs under any \eeldpe{} mechanism satisfying \cref{def:noninteractive-dp}.
    \item \textbf{Bounding via $\chi^2$-divergence:} We bound the total variation metric from above using the $\chi^2$-divergence and simplify the expression in terms of the induced marginal distributions of the private outputs.
    \item \textbf{Simplifying $\chi^2$-divergence:} We further simplify the expression for the $\chi^2$-divergence using the structure of the probability matrices generated by the latent positions in $\calX_0 \cup \calX_1$ and the constraints imposed by \eeldpe{}. The proof in this step is fairly technical but uses standard techniques.
    \item \textbf{Setting the optimal separation $\eta$:} We set the optimal $\eta$ based on the above steps to obtain the result.
\end{enumerate}

\quad\textbf{Step \Circled{1}.}\quad Let $\calX$ be the set of all $(p,q)$-admissible latent positions, and fix $\Xo \in \calX$ as follows. Let $\Uo \in \bbU(n, d)$ be the matrix guaranteed by \cite[Lemma~4]{yan2023minimax} satisfying:
\begin{align}
    \Uo\onev[d] = \frac{1}{\sqrt{n}}\onev[n] \qc{} \sqrt{\frac{d}{n}} \le \ttinf{\Uo} \le \sqrt{\frac{d}{n-2d}}, \qq{and} \max_{i \in [n], j \in [d]} \abs{\U_{0,ij}} \le \frac{1}{\sqrt{n-2d}}.
\end{align}
i.e., the first column of $\Uo$ is proportional to the vector of $1$s and has incoherence parameter $\mu(\Uo) = nd/(n-2d)$. For $\lambda_1 = {n/2}$ and $\lambda_2 = \cdots = \lambda_d = {n/12d}$, let $\L = \textup{diag}(\lambda_1, \dots, \lambda_d)$, and set $\Xo \defeq \Uo\L\hpow$. By construction, when $n \ge 4d$ it is easy to see that $\Xo \in \calX$ since the rows $\xv_{0, 1}, \dots, \xv_{0, n} \in \R^d$ of $\Xo$ are such that
\begin{align}
    \xv_{0, i}\tr \I[p, q] \xv_{0, j} &= \sum_{\ell} \omega_\ell \lambda_\ell u_{i\ell} u_{j\ell} \le \lambda_1\frac{1}{n} + \sum_{\ell=2}^d \lambda_\ell \abs{u_{i\ell}} \abs{u_{j\ell}} \le \frac{1}{2} + \frac{n}{12(n-2d)} \le \frac{2}{3}\label{eq:2/3}\\
    \xv_{0, i}\tr \I[p, q] \xv_{0, j} &= \sum_{\ell} \omega_\ell\lambda_\ell u_{i\ell} u_{j\ell} \ge \lambda_1\frac{1}{n} - \sum_{\ell=2}^d \lambda_\ell \abs{u_{i\ell}} \abs{u_{j\ell}} \ge \frac{1}{2} - \frac{n}{12(n-2d)} \ge \frac{1}{3}\label{eq:1/3},
\end{align}
where $\omega_\ell \in \qty{-1, +1}$ are the diagonal entries of $\I[p,q]$. Let $\Po = \Xo\I[p,q]\Xo\tr$ be the probability matrix with sparsity $\rho_n \le 1$ and take $\calX_0 = \qty{\Xo}$. Consider the following set of $n$ latent positions:
\begin{align}
    \calX_1 = \qty{ \X_k \in \Rnd: \X_k = \Xo + \eta e_k\vv(k)\tr\qc{} \vv(k) = \xv_{0,k}/\norm{\xv_{0,k}} \in \bbS^{d-1} }.
\end{align}
In other words, $\X_k$ is obtained from $\Xo$ by scaling the $k$-th row of $\Xo$ by $\eta$. It is easy to see that $\X_k \in \calX$ for all $\eta<\sqrt{3/2}-1$ sufficiently small. By construction, we have $\dttinf(\X_k, \Xo) \ge \eta/2$. Let $\P_k = \X_k\I[p,q]\X_k\tr$ be the probability matrix generated by each $\X_k$ with sparsity $\rho_n \le 1$, and define $\G_k$ to be the difference between the probability matrices $\P_k$ and $\Po$ given by
\begin{align}
    \G^k = \P_k - \P_0 =  \rho_n\qty\Big(\eta\uv(k) e_k\tr + \eta e_k\uv(k)\tr + \eta^2 e_k e_k\tr),\qq{for $\uv(k) = \Xo\I[p,q]\vv(k)$.}\label{eq:gamma-def}
\end{align}

\quad\textbf{Step \Circled{2}.}\quad Consider the GRDPGs $\Av_0 \sim \grdpg(\X_0, \rho_n; p, q)$ and $\Av_k \sim \grdpg({\X_k, \rho_n}; p, q)$, i.e., $\Av_0 \sim \textup{Ber}(\P_0)$ and $\Av_k \sim \textup{Ber}(\P_k)$, respectively, and let $F_0$ and $F_k$ denote their joint distributions on $\bbB(n)$. For any \eeldpe{} mechanism $\A \in \bbA_\eps$ satisfying \cref{def:noninteractive-dp}, let $\Zv_0 = \A(\Av_0)$ and $\Zv_k = \A(\Av_k)$ be the private outputs with distributions $\A\push F_0$ and $A\push F_k$ on $\calZ$, respectively, i.e., for all $\ell \in \qty{0} \cup [n]$ and $A = (a_{ij}) \in \bbB(n)$,
\begin{align}
    F_\ell(A) = \prod_{i < j}F_{\ell, ij}(a_{ij}) = \prod_{i < j}\Pv_{\ell, ij}^{a_{ij}}(1 - \Pv_{\ell,ij})^{1-a_{ij}}
\end{align}  
and for $Z = (z_{ij})$,
\begin{align}
    \A\push F_\ell(Z) = \int_{A \in \bbB(n)} \calQ(Z \mid \Av = A) dF_k(A) = \prod_{i < j}\int_{a_{ij} \in \qty{0,1}} \calQ_{ij}(z_{ij} \mid \Av_{ij} = a_{ij}) dF_{\ell, ij}(a_{ij}) =: \psi_{\ell, ij}(z_{ij}).\label{eq:psi-def}
\end{align}
Let $\overline{\A\push F} = \frac{1}{n}\sum_k\A\push F_k \in \text{co}(\qty{\A\push F_1, \dots, \A\push F_n})$ be the mixture distribution in the convex hull of $\{\A\push F_1, \dots, \A\push F_n\}$. From Le Cam's lemma \cite[Lemma~1]{yu1997assouad}, it follows that
\begin{align}
    \rn(\calX, \eps) \defeq \inf_{\A_\eps \in \mathbb{A}_\epsilon}\inf_{\hX} \sup_{\Xv \in \calX } {\E}\qty[\dttinf\qty(\hX(\A_\eps(\Av)), \X)] \ge \inf_{\A_\eps \in \mathbb{A}_\epsilon} \frac{\eta}{4} \qty\Big(1 - \tv\qty(\overline{\A\push F}, \A\push F_0)),\label{eq:le-cam}
\end{align}
where $\tv(\cdot, \cdot)$ is the total variation metric. 

\quad\textbf{Step \Circled{3}.}\quad Let $\varphi_0$ and $\varphi_k$ denote the density functions associated with the measures $\A\push F_0$ and $A\push F_k$ on $\calZ$, i.e., ${d\A\push F_k}/{d\A\push F_0} = {\varphi_k}/{\varphi_0}$, and let $\bar\varphi = \frac{1}{n}\sum_k\varphi_k$. From the properties of the Chi-squared divergence $\chisq(\cdot, \cdot)$ in \cite[Lemma~2.27 and Eq.~(2.15)]{tsybakov2008nonparametric}, we get
\begin{align}
    \tv(\overline{\A\push F}, \A\push F_0)^2\label{eq:tv-chisq}
    &\le \chisq(\overline{\A\push F}, \A\push F_0)\\
    &= \int \qty(\frac{d\A\push F_k}{d\A\push F_0})^2 d\A\push F_0 - 1\\
    &= \int \qty(\frac{\bar\varphi}{\varphi_0})^2 d\A\push F_0 - 1 = \E\qty[\qty(\frac{\bar\varphi(\Zv_0)}{\varphi_0(\Zv_0)})^2] - 1.
\end{align}
where the expectation is taken over $\Zv_0 \sim \A\push F_0$. Moreover, by expanding the square and using \cref{eq:psi-def}, we have
\begin{align}
    \qty(\frac{\bar\varphi(Z)}{\varphi_0(Z)})^2 = \frac{1}{n^2} \sum_{k, \ell}\frac{\varphi_k(Z)\varphi_\ell(Z)}{\varphi_0(Z)^2} = \frac{1}{n^2} \sum_{k, \ell} \prod_{i < j}\frac{\psi_{k, ij}(z_{ij})\psi_{\ell, ij}(z_{ij})}{\psi_{0, ij}(z_{ij})\psi_{0, ij}(z_{ij})},
\end{align}
and it follows that
\begin{align}
    1 + \chisq(\overline{\A\push F}, \A\push F_0) = \frac{1}{n^2}\sum_{k, \ell}\prod_{i < j} \E_{\Zv_{0, ij}}\qty[\frac{\psi_{k, ij}(\Zv_{0,ij})\psi_{\ell, ij}(\Zv_{0,ij})}{\psi_{0, ij}(\Zv_{0,ij})\psi_{0, ij}(\Zv_{0,ij})}] = \frac{1}{n^2}\sum_{k, \ell}\prod_{i < j} \int \frac{\psi_{k, ij}(z_{ij})\psi_{\ell, ij}(z_{ij})}{\psi_{0, ij}(z_{ij})} dz_{ij}.\label{eq:chisq-bound}
\end{align}

\quad\textbf{Step \Circled{4}.}\quad The simplification in this step is similar to the steps in \cite[Appendix~B]{li2022network}. Consider the integral in \cref{eq:chisq-bound}. We begin by noting that
\begin{align}
    \psi_{k, ij}(z) 
    &= \int_{a_{ij} \in \qty{0,1}} \calQ_{ij}(z \mid \Av_{ij} = a_{ij}) dF_{k, ij}(a_{ij})\\ 
    &= \calQ_{ij}(z \mid \Av_{ij} = 1) \Pv_{k, ij} + \calQ_{ij}(z \mid \Av_{ij} = 0) (1 - \Pv_{k, ij})\\
    &= \calQ_{ij}(z \mid \Av_{ij} = 1) (\Pv_{0, ij} + \G_{k, ij}) + \calQ_{ij}(z \mid \Av_{ij} = 0) (1 - \Pv_{0, ij} - \G_{k, ij})\\
    &= (\P_{0, ij} + \G_{k, ij})\qty\Big(\calQ_{ij}(z \mid \Av_{ij} = 1) - \calQ_{ij}(z \mid \Av_{ij} = 0)) + \calQ_{ij}(z \mid \Av_{ij} = 0)\\
    &= (\P_{0, ij} + \G_{k, ij})\qty\Big(\xi_1(z) - \xi_0(z)) + \xi_0(z),
\end{align}
where we used the fact that $\P_{k, ij} = \P_{0, ij} + \G_{k, ij}$ and defined $\xi_1(z) \defeq \calQ_{ij}(z \mid \Av_{ij} = 1)$ and $\xi_0(z) \defeq \calQ_{ij}(z \mid \Av_{ij} = 0)$. Plugging in the above expression into the integral in \cref{eq:chisq-bound} and simplifying, we get
\begin{align}
    \frac{\psi_{k, ij}(z)\psi_{\ell, ij}(z)}{\psi_{0, ij}(z)} 
    &= \frac{\qty[(\P_{0, ij} + \G_{k, ij})\qty\Big(\xi_1(z) - \xi_0(z)) + \xi_0(z)] \cdot \qty[(\P_{0, ij} + \G_{\ell, ij})\qty\Big(\xi_1(z) - \xi_0(z)) + \xi_0(z)]}{\P_{0, ij}\qty\Big(\xi_1(z) - \xi_0(z)) + \xi_0(z)}\\
    &= \overbrace{\qty[\P_{0, ij}\qty\Big(\xi_1(z) - \xi_0(z)) + \xi_0(z)]}^{=: f_1(z)} + \overbrace{\qty\Big[(\G_{k, ij} + \G_{\ell, ij})(\xi_1(z) - \xi_0(z))]}^{=:f_2(z)}\\ 
    &\qq{}\qq{}+ \underbrace{\qty[\G_{k, ij}\G_{\ell, ij}\frac{(\xi_1(z) - \xi_0)^2}{\P_{0, ij}\qty\Big(\xi_1(z) - \xi_0(z)) + \xi_0(z)}]}_{=:f_3(z)}.
\end{align}
Note that
\begin{align}
    \int f_1(z) dz &= \P_{0, ij}\int\xi_1(z)dz + (1 - \P_{0, ij})\int\xi_0(z)dz = 1\\ 
    \int f_2(z)dz &= (\G_{k, ij} + \G_{\ell, ij}) \int \qty\big(\xi_1(z) - \xi_0(z))dz = 0,
\end{align}
and,
\begin{align}
    \int f_3(z)dz = \G_{k, ij}\G_{\ell, ij}\int \frac{(\xi_1(z) - \xi_0(z))^2}{\P_{0, ij}\qty\Big(\xi_1(z) - \xi_0(z)) + \xi_0(z)}dz.
\end{align}

From \cref{eq:2/3} we have $\P_{0,ij} \le {2\rho_n}/{3}$, and for $\eps \le \min\qty{1, 3/8\rho_n} \le \min\qty{1, 1/4\P_{0, ij}}$ using \cite[Lemma~B.1]{li2022network}, we have
\begin{align}
    \P_{0, ij}\qty\Big(\xi_1(z) - \xi_0(z)) + \xi_0(z) \ge \xi_0(z)/2.
\end{align}
Plugging this back into the integral of $f_3(z)$,
\begin{align}
    \int f_3(z)dz \le \G_{k, ij}\G_{\ell, ij} \int\frac{\xi_0(z)^2 (\xi_1(z)/ \xi_0(z) - 1)^2}{\xi_0(z)/2}dz 
    &= 2\G_{k, ij}\G_{\ell, ij} \int \qty( \frac{\xi_1(z)}{\xi_0(z)}-1 )^2 \xi_0(z)dz\\ 
    &\le 2\G_{k, ij}\G_{\ell, ij}(e^{\eps} - 1)^2,
\end{align}
where the final inequality follows from the definition of \eeldpe{} in \cref{def:noninteractive-dp} which requires that
\begin{align}
    \frac{\xi_1(z)}{\xi_0(z)} = \frac{\calQ_{ij}(z \mid \Av_{ij} = 1)}{\calQ_{ij}(z \mid \Av_{ij} = 0)} \le e^{\eps}.
\end{align}
Since $\eps \le 1$, it follows that $(e^{\eps}-1) = \se^2(e^{\eps}+1) \le 4\se^2$, and
\begin{align}
    \int \frac{\psi_{k, ij}(z)\psi_{\ell, ij}(z)}{\psi_{0, ij}(z)} dz \le 1 + 2(e^{\eps} - 1)^2 \le 1+ c_1\G_{k, ij}\G_{\ell, ij} \cdot  \se^4 \le \exp(c_1\G_{k, ij}\G_{\ell, ij} \cdot  \se^4),
\end{align}
where $c_1 = 16$. Plugging this bound back into \cref{eq:chisq-bound}, we get
\begin{align}
    1 + \chisq(\overline{\A\push F}, \A\push F_0) 
    &= \frac{1}{n^2}\sum_{k, \ell}\prod_{i < j} \int \frac{\psi_{k, ij}(z)\psi_{\ell, ij}(z)}{\psi_{0, ij}(z)} dz\\ 
    &\le \frac{1}{n^2}\sum_{k, \ell}\prod_{i < j} \exp\qty(c_1\se^4\G_{k, ij}\G_{\ell, ij} \cdot  \se^4)\\
    &= \frac{1}{n^2}\sum_{k, \ell}\exp\qty(c_1\se^4\sum_{i < j}\G_{k, ij}\G_{\ell, ij})\\
    &= \frac{1}{n^2}\sum_{k, \ell}\exp\qty\Big(c_1\se^4\inner{\G_k, \G_\ell}_F)\\
    &= \frac{1}{n^2}\sum_{k}\exp\qty\Big(c_1\se^4\norm{\G_k}^2_F) + \frac{1}{n^2}\sum_{k \neq \ell}\exp\qty\Big(c_1\se^4\inner{\G_k, \G_\ell}_F)\\
    &\le \frac{1}{n}\exp\qty\Big(c_1\se^4\max_k\norm{\G_k}^2_F) + \frac{n(n-1)}{n^2}\exp\qty\Big(c_1\se^4 \max_{k \neq \ell}\inner{\G_k, \G_\ell}_F)\\
    &\le \exp\qty\Big(c_1\se^4\max_k\norm{\G_k}^2_F - \log{n}) + \exp\qty\Big(c_1\se^4 \max_{k \neq \ell}\inner{\G_k, \G_\ell}_F).\label{eq:chisq-bound-intermediate}
\end{align}

\quad\textbf{Step \Circled{5}.}\quad For $\G_k$ in \cref{eq:gamma-def}, using the triangle inequality and by noting that $\uv(k) = \Xo\I[p,q]v(k)$, we have
\begin{align}
    \norm{\G_k}_F 
    &\le \rho_n\norm{\eta\uv(k) e_k\tr + \eta e_k\uv(k)\tr + \eta^2 e_k e_k\tr}_F\\
    &\le \rho_n\qty(\eta^2 + 2\eta \norm{u(k)})\\
    &\le \rho_n(\eta^2 + 2\eta \opnorm{\Xo}) \le 4\rho_n\eta\sqrt{n},
\end{align}
since $\eta^2 < \eta < 1$. It follows that
\begin{align}
    \max_k\norm{\G_k}_F^2 \le 16\rho_n^2\eta^2n.
\end{align}
Similarly, note that $\G_k$ has non-zero entries only in the $k$-th row and column, and for $k \neq \ell$ we have a total of two non-zero entries in $\G_k\tr\G_\ell$, and, therefore,
\begin{align}
    \inner{\G_k, \G_\ell}_F = \textup{tr}(\G_k\tr\G_\ell) = 2\rho_n^2\eta^2 \cdot u(k)_\ell \cdot u(\ell)_k \le 2\rho_n^2\eta^2,
\end{align}
where $u(k)_\ell = \xv_{0, \ell}\tr\I[p,q]\vv(k) \le 1$ and similarly for $u(\ell)_k$. Plugging these bounds back into $\chisq(\cdot, \cdot)$ in \cref{eq:chisq-bound-intermediate}, we get
\begin{align}
    1 + \chisq(\overline{\A\push F}, \A\push F_0) \le \exp\qty\Big(16 c_1 \se^4\rho_n^2\eta^2n - \log{n}) + \exp\qty\Big(8 c_1 \se^4\rho_n^2\eta^2).
\end{align}
For a suitably large absolute constant $C > 16c_1$, choosing
$$
\eta = \frac{1}{C}\sqrt{\frac{\log{n}}{n\rho_n^2\se^4}}
$$ 
results in $\chisq(\overline{\A\push F}, \A\push F_0) = o(1)$, and from \cref{eq:tv-chisq} and Le Cam's lemma in \cref{eq:le-cam}, we have
\begin{align}
    \rn(\calX, \eps) \ge \frac{\eta}{4} \qty(1 - \sqrt{\chisq(\overline{\A\push F}, \A\push F_0)}) = \Omega\qty(\sqrt{\frac{\log{n}}{n\rho_n^2\se^4}}).
\end{align}
This completes the proof of \cref{thm:minimax}.
\qed
\endgroup

\begingroup
\newcommand{\rhn}{\rho_n^{1/2}}
\newcommand{\chrn}{\ch{\rho}_n^{1/2}}
\newcommand{\srn}{\sigma^2\rho_n}
\newcommand{\srnh}{\sigma\rho^{1/2}_n}
\newcommand{\tn}{t_n}
\newcommand{\Qtn}{\Q_{\sqrt{\tn}\X}}
\newcommand{\Qsrnh}{\Q_{\srnh\X}}
\newcommand{\Lt}{\overline{\L}}
\renewcommand{\Avt}{\overline{\Av}}
\renewcommand{\Pt}{\overline{\P}}
\renewcommand{\Ut}{\overline{\U}}
\renewcommand{\Xt}{\overline{\X}}

\subsection{Proof of \cref{thm:convergence-rate}}
\label{proof:prop:grdpg-estimation}

First, let $\Avc$ be the \textit{privacy-adjusted adjacency matrix} from \cref{alg:pase} and let $\Avt$ be the matrix given by
\eq{
    \Avt = \M_\e(\Av) - \tau^2 \onev\onev\tr.
}
with spectral decomposition $\Avt = \Ut {\Lt} \Ut\tr$ corresponding to the leading $d$ eigenvalues by absolute value. Let ${\Xt = \Ut \qty\big|{ \Lt }|\hpow \in \R^{n \times d}}$ be the spectral embedding of $\Avt$. The proof outline is as follows:

\begin{enumerate}[label=\Circled{\arabic*}, itemsep=0.0em]
    \item We first show that $\Avt$ is close to the expected adjacency matrix $\Pt$ of a GRDPG, and establish the convergence rate of $\Xt$ to $\tn\X$ in the $\dttinf$ metric. 
    \item We use the result in \Circled{1} to quantify the convergence rate of $\rho_n\hpow\Xc$ to $\X$ in the $\dttinf$ metric.
    \item Finally, use $\ch{\rho}_n$ from \cref{alg:pase} as a plug-in estimator to get the final bound.
\end{enumerate}

\textbf{Step~\Circled{1}}.\quad In a similar fashion to the sparse GRDPGs considered in \citet[Theorem~5]{rubin2022statistical}, let $\Pt$ be the expected adjacency matrix for the random dot product graph $(\Bv, \s\X) \sim \grdpg\pa{\s\push\pr, \rho_n; p, q}$ given by
\eq{
    \Pt = \rho_n \cdot (\s\X) \I[p,q] (\s \X)\tr = (\tnh\X)\I[p,q](\tnh\X)\tr,\nn
}
where, for notational simplicity, $\tnh \defeq \srnh$ denotes the \textit{effective sparsity parameter} under privacy. Although $\Avt$ is not the adjacency matrix of a GRDPG, it is close to $\Bv$. Indeed, for all $1 \le i < j \le n$,
\eq{
    \E(\Avt_{ij}) = \E\qty( \M_\e(\Av)_{ij} - \tau^2 ) 
    &\en{i} Y_i\tr\I[p,q]Y_j - \tau^2\nn\\
    &\en{ii} (\s\rhn X_i)\tr\I[p,q](\s\rhn X_j) + \tau^2 - \tau^2\nn\\
    &= \s^2\rho_n X_i\tr\I[p,q]X_j\nn\\
    &= \Pt_{ij},\nn
}
where in (i) we used the fact that $(\M_\e(\Av), \Y) \sim \grdpg\pa{\varphi_\sharp\pr, 1; p+1, q}$ from Theorem~\ref{prop:closure} and (ii) follows from the definition of $\varphi$. Using \cite[Theorem~5.2]{lei2015consistency} (see \cref{lemma:lei1}), with probability $1 - O(n^{-1})$ it follows that
\eq{
    \norm{\Avt - \Pt} = \norm{\Av_\Y - \P_\Y} = O(\sqrt{n}).\label{eq:lei-1}
}
Furthermore, for $\Ut$ in the spectral decomposition $\Pt = \Ut \Lt \Ut\tr$, from\footnote{Note that the $\log^c{n}$ dependence in \cite[Lemma~12]{rubin2022statistical} is replaced with $\log{n}$ using \cite[Theorem~4]{agterberg2020two}.} \cite[Lemma~12]{rubin2022statistical}, it follows that
\eq{
    \ttinf{(\Avt - \Pt)\Ut} = O\qty(\frac{\sqrt{n} \log n}{\sqrt{n}}),\nn
}
with probability $1 - O(n^{-1})$. Next, we prepare to quantify $\dttinf(\Xt, \tn\X)$. In order to do so, we need to characterize the matrices which align $\Xt$ to $\tn\X$ as per \cref{def:dttinf}. 

To this end, let $\Qtn$ be the alignment matrix satisfying $\Ut|\Lt|\hpow = \tnh\X\Qtn\inv$ (see \cref{fact:Qx}); by invariance of $\qx$ to scale transformations (\cref{lemma:qx-scale}), it follows that $\Qtn = \qx$. Let $\Qn = \ox\tr\Qtn=\ox\tr\qx$ where $\ox \in \o(d)$ is described in \cref{fact:Ox}. From \citet[Eq.~15]{rubin2022statistical}, it follows that
\eq{
    \Xt - \tn\hpow\X\Qn\inv &= (\Avt - \Pt)\Ut |{\Lt}|\pow \I[p,q]\ox + \Rv.
}
From the definition of $\dttinf$ from \cref{eq:dttinf-2} in \cref{def:dttinf}, 
\begin{align}
    \dttinf(\Xt, \tn\X) \le \ttinf{\Xt - \tn\hpow\X\Qn\inv}.
\end{align}

From \cite[Eq.~(14)]{rubin2022statistical} and \cite[Lemma~3]{agterberg2020nonparametric}, the matrix $\Rv$ is such that with probability $1 - O(n^{-1})$,
\eq{
    \ttinf{\Rv} = \op\qty(\sqrt{\f{\log^4 n}{n^2 \tn}}),\nn
}
and, for the first term on the r.h.s., using \citet[Propositions~6.3~\&~6.5]{cape2019two}, we obtain
\eq{
    \ttinf{(\Avt - \Pt)\Ut |{\L({\Pt})}|\pow \I[p,q]} &\le \ttinf{(\Avt - \Pt)\Ut} \norm{|{\L({\Pt})}|\pow \I[p,q]}\nn\\
    &{=} \ttinf{(\Avt - \Pt)\Ut} \norm{|{\L({\Pt})}|\pow} \norm{\I[p,q]}\nn\\
    &{\le} \ttinf{(\Avt - \Pt)\Ut} \norm{|{\L({\Pt})}|\pow}\nn\\
    &\num{i}{=}O\qty(\f{\log n}{\sqrt{nt_n}}),\nn
}
with probability $1 - O(n^{-1})$, where (i) uses the fact that $\eigen[i](\Pt) = \Theta(nt_n)$. Plugging in the value for $\tn$, we obtain
\begin{align}
    \ttinf{\Xt - \srnh\X\Qn\inv} = O\qty(\f{\log{n}}{\sqrt{n \sigma^2\rho_n}} + \sqrt{\f{\log^4 n}{n^2 \srn}}) = O\qty(\f{\log n}{\sqrt{n\sigma^2\rho_n}}).\label{eq:srn-convergence}
\end{align}

\textbf{Step~\Circled{2}}.\quad Note that $\Ac = \frac{1}{\s^2}\Avt$, and the spectral decomposition of $\Ac$ corresponding to the leading $d$ eigenvalues by magnitude becomes
$$
\Ac = \ch\U \ch{\L} \ch\U\tr = \Ut \qty( \frac{1}{\s^2} \Lt ) \Ut\tr.
$$
Therefore, $\Xc = \ch\U\abs\big{\ch\L}\hpow = \frac{1}{\s}\Xt$, and from \cref{eq:srn-convergence}, it follows that 
\begin{align}
    \ttinf{ \Xc - \rho_n\hpow\X\Qn\inv} = O\qty( \frac{\log{n}}{\sqrt{n\s^4\rho_n}} )\qq{and} \ttinf{ \rho_n\pow\Xc - \X\Qn\inv } = O\qty( \frac{\log{n}}{\sqrt{n\s^4\rho_n^2}} ),
    \label{eq:convergence-rate-1}
\end{align}
with probability $1 - O(n^{-1})$.

\textbf{Step~\Circled{3}}.\quad Let $\Zv = \M_\e(\Av)$ and consider the estimator $\ch{\rho}_n$, given by
\begin{align}
    \ch{\rho}_n = {n \choose 2}\inv\sum_{i < j} \Ac_{ij} = \frac{1}{\s^2}\qty{ {n \choose 2}\inv\sum_{i < j} \Zv_{ij} - \tau^2 }= \frac{1}{\s^2}(\hat\theta - \tau^2).
\end{align}
Since the entries $\M_{ij}$ are \iid{}, $\hat\theta \defeq {n \choose 2}\inv \sum_{i < j}\Mv_{ij}$ is a second-order $U$-statistic. The expected value of $\hat\theta$ is given by
\begin{align}
    \E(\hat\theta \mid \X) = \E( \Zv_{ij} \mid \X) = Y_i\tr\I[p+1, q]Y_j = \rho_n\s^2 (X_i\tr \I[p, q]X_j) + \tau^2,
\end{align}
and, by assumption that $\E[ X_i\tr\I[p,q]X_j ] = 1$, it follows that
\begin{align}
    \E(\hat\theta) = \E_\X\qty[\E(\hat\theta \mid \X) ] = \rho_n\s^2 \E[ X_i\tr\I[p,q]X_j ] + \tau^2 = \rho_n\s^2 + \tau^2.
\end{align}
Therefore, 
$$
\E\qty( \ch{\rho}_n \mid \X ) = \rho_n \cdot {n \choose 2}\inv\sum_{i < j}\P_{ij} \qq{and} \E\qty( \ch{\rho}_n ) = \rho_n.
$$
By the standard protocol for bounding  $U$-statistics by using Hoeffding's inequality (see, e.g., \cite[Lemma~5]{agterberg2020nonparametric}), it follows that with probability $1 - O(n^{-2})$,
\begin{align}
    \frac{1}{\sqrt{\ch{\rho}_n}} = \frac{1}{\sqrt{{\rho}_n}} \qty( 1 + O\qty(\sqrt{\frac{\log{n}}{n\rho_n}}) ). \label{eq:ch-rho-n}
\end{align}
Combining the results from \cref{eq:convergence-rate-1} and \cref{eq:ch-rho-n}, we get that when $n\rho_n = \omega(\log{n})$,
\begin{align}
    \ttinf{ \ch{\rho}_n\pow\Xc - \X\Qn\inv } = \ttinf{ {\rho}_n\pow\Xc - \X\Qn\inv } \qty\Big(1+o(1)) = O\qty( \frac{\log{n}}{\sqrt{n\s^4\rho_n^2}} ).\label{eq:convergence-rate-2}
\end{align}
Since $\Qn\inv = \Qx\inv\ox$, it follows that with probability $1 - O(n^{-1})$,
\begin{align}
    \dttinf\qty( \ch{\rho}_n\pow\Xc, \X ) \le \ttinf{ \ch{\rho}_n\pow\Xc - \X\Qn\inv } = O\qty( \frac{\log{n}}{\sqrt{n\s^4\rho_n^2}} ). \label{eq:convergence-rate-3}
\end{align}
\QED
\endgroup

\begingroup
\newcommand{\Dn}{\sfD_n}
\newcommand{\cDn}{\ch{\sfD}_n}
\subsection{Proof of \cref{thm:bottleneck-convergence}}
\label{proof:thm:bottleneck-convergence}

Let $\Xc, \ch{\rho}_n$ be the privacy-adjusted spectral embedding and the estimated sparsity parameter from \cref{alg:pase}, respectively, and let $\Qx, \qz, \ox, \oz$ be as defined in \cref{fact:Qx}, \cref{fact:Qz} and \cref{fact:Oz}, respectively, and let $\Qn = \ox\tr\Qx$ as in the proof of \cref{thm:convergence-rate} in \cref{proof:prop:grdpg-estimation}. 

Given $\sfD_n = \dgm(\X)$ and $\ch{\sfD}_n = \dgm(\Xc/\sqrt{\ch{\rho}_n})$, by the triangle inequality we have
\begin{align}
    \Winf( \cDn, \Dn ) 
    &\le \Winf\qty( \cDn, \dgm(\X\Qn\inv) )\\ 
    &\qquad+ \Winf\qty( \dgm(\X\Qn\inv), \dgm(\X\Qx\inv) )\\ 
    &\qquad+ \Winf\qty( \dgm(\X\Qx\inv), \dgm(\X\Qz\inv\oz) )\\ 
    &\qquad+ \Winf\qty( \dgm(\X\Qz\inv\oz), \Dn ).
\end{align}
By the stability of persistence diagrams (\cref{thm:stability}), for $\X, \Y \in \Rnd$,
$$
\Winf(\dgm(\X), \dgm(\Y)) \le d_H(\X, \Y) = \ttinf{\X - \Y},
$$
where $d_H(\X, \Y)$ denotes the Hausdorff distance between $\X$ and $\Y$. Therefore, we have the following bounds:
\begin{align}
\Winf( \cDn, \Dn ) 
    &\le \ttinf\big{\ch{\rho}_n\Xc - X\Qn\inv} \tag{$=:T_1$}\\ 
    &\qquad+ \Winf\qty( \dgm(\X\Qn\inv), \dgm(\X\Qx\inv) )\tag{$=:T_2$}\\ 
    &\qquad+ \ttinf\big{ \X\Qx\inv - \X\Qz\inv\oz }\tag{$=:T_3$}\\ 
    &\qquad+ \Winf\qty( \dgm(\X\Qz\inv\oz), \Dn )\tag{$=:T_4$}
\end{align}

The proof proceeds by bounding each term $T_1, T_2, T_3, T_4$ in the following steps. From \cref{thm:convergence-rate}, and in particular, from \cref{eq:convergence-rate-2}, it follows that with probability $1 - O(n^{-1})$,
\begin{align}
    T_1 = \ttinf{\ch{\rho}_n\Xc - \X\Qn\inv} = O\qty( \frac{\log{n}}{\sqrt{n\s^4\rho_n^2}} ).
\end{align}

For the second term, $\Qn\inv = \Qx\inv\ox$  where $\ox \in \o(d)$. Since persistence diagrams are invariant to orthogonal transformations \cref{eq:pd-invariance}, it follows that
\begin{align}
    T_2 = 0
\end{align}

Similarly, for the fourth term, since $\qz \in \o(p, q) \cap \o(d)$ (see \cref{fact:Qz}) and $\oz \in \o(d)$, it follows that $\qz\inv\oz \in \o(d)$. By invariance to $\o(d)$ transformations, it also follows that 
$$
T_4 = 0.
$$

Finally, for the third term, using \cite[Propositions~6.3~\&~6.5]{cape2019two},
\begin{align}
    T_3 = \ttinf\big{ \X\Qx\inv - \X\Qz\inv\oz } = \ttinf{\X} \opnorm\big{\Qx\inv - \X\Qz\inv\oz} = \op\qty( \frac{\log{n}}{\sqrt{n}} ),
\end{align}
where the final bound follows from \cref{lemma:subspace-alignment} and the fact that $\ttinf{\X} \le C$ for some constant $C$ when $\X$ is $(p,q)$-admissible \cite[Theorem~3]{agterberg2020nonparametric}.

Combining the bounds for $T_1, T_2, T_3, T_4$, 
\begin{align}
    \Winf( \cDn, \Dn ) = O\qty( \frac{\log{n}}{\sqrt{\s^2n\rho_n}} ),
\end{align}
and we obtain the desired result. \QED
\endgroup



\section{Topological Data Analysis}
\label{sec:tda-appendix}

We present the necessary background for Topological Data Analysis (TDA) here. We refer the reader to \cite{wasserman2018topological,chazal2017introduction} for an accessible overview and to \cite{edelsbrunner2010computational} for a comprehensive treatment.

\textbf{{Persistence Diagrams.}} Given a collection of points $\X \in \Rnd$ where $\pb{\xv_1, \dots, \xv_n} \subset \R^d$, persistent homology sheds important insights on the geometric and topological features underlying $\X$. The underlying philosophy is as follows. The shape of $\X$ at resolution $t > 0$ is encoded in the union of balls $B(\X, t) \defeq \cup_{\xv \in \X}B(\xv, t)$ centered at the sample points $\X$. Typically, this information is extracted using a geometric object $\k(\X, t)$, called the \textit{simplicial complex}.

The collection $\pb{B(\X, t) : {t \ge 0}}$, called a \textit{filtration}, is a sequence of nested topological spaces, i.e. for $t_1\!<\!t_2\!<\!\dots\!<\!t_m$, $B(\X, t_1) \subseteq B(\X, t_2) \subseteq \dots \subseteq B(\X, t_m)$. As $t$ varies, the evolution of topology is encoded in this filtration. Roughly speaking, new cycles (i.e. connected components, loops, holes, etc.) appear, or existing cycles disappear (get filled out). Specifically, each $k$-dimensional feature in $B(\X, t)$ can be represented as an element in a vector space, $H_k(\X, t)$, called its \textit{homology}. When a new $k$-dimensional feature appears in $B(\X, b)$ for some $b>0$, then a non-trivial $k$-cycle appears in $H_k(\X, b)$, and the dimension of this vector space increases by one. In this case, the feature is said to be \textit{born} at $b$. The same $k$-dimensional feature \textit{dies} at resolution $d > 0$, if this $k$-cycle is absent in $H_k(\X, d+\delta)$, for all $\delta > 0$. \textit{Persistent homology}, $\textbf{PH}_*(\X)$, is an \textit{algebraic module} which tracks these persistent pairs $(b, d)$ of births and deaths across the entire filtration for every $k$. The persistent pairs in $\textbf{PH}_*(\X)$ are summarized in a \textit{persistence diagram}
$$
\dgm\pa{\X} = \pb{(b, d)\in\R^2: 0 \le b < d \le \infty \qq{and} (b, d) \in \textbf{PH}_*(\X)}.
$$
See Figure~7 for an illustration. The space of persistence diagrams $\Omega = \pb{(x, y) \in \R^2: x < y}$, is endowed with a collection of Wasserstein metrics $\pb{W_p(\cdot, \cdot)}_{p \ge 1}$, and the special case of $W_\infty(\cdot, \cdot)$ is referred to as the \emph{bottleneck distance}. 

\begin{definition}
    For two persistence diagrams $\mathsf{D}_1$ and $\mathsf{D}_2$, the bottleneck distance is~given~by
    $$
    \w\pa{ \mathsf{D}_1, \mathsf{D}_2 } \defeq \inf\limits_{\phantom{\int^{1}_{1}} \phi \in \Phi \phantom{\int^{1}_{1}}}\!\!\!\!\sup\limits_{\xv \in \mathsf{D}_1 \cup \partial\Omega} \norminf{\xv - \phi(\xv)},
    $$
    where $\Phi$ is the collection of bijections between $\mathsf{D}_1$ and $\mathsf{D}_2$ including the boundary $\partial\Omega = \pb{(x, y)\in \R^2: x=y}$ with infinite multiplicity. 
\end{definition}

We assume throughout this work that the probability distribution $\pr$ underlying the random dot product graph has compact support. Compactness ensures that the resulting persistence diagrams are \textit{pointwise finite dimensional}, and satisfy a stability property w.r.t. the Hausdorff distance $d_H(\cdot, \cdot)$.

\begin{theorem}[{Stability of Persistence Diagrams, \citealp{chazal2016structure}}]\label{thm:stability} 
    For compact sets $\bbX, \bbY \subset (\mathbb{M}, \rho)$ from a metric space, 
$$
    \w\pa{ \dgm(\bbX), \dgm(\bbY) } \le d_H(\bbX, \bbY) \defeq \inf_{\xv \in \bbX}\sup_{\yv \in \bbY}\rho(\xv, \yv).
$$
and between two finite point clouds $\X \in \Rnd$ and $\Y \in \R^{m \times d}$,
$$
\w\pa{ \dgm(\X), \dgm(\Y) } \le d_H(\X, \Y) \defeq \max\qty{\min_{i \in [n]}\max_{j \in [m]}\norm{\xv_i - \yv_j}, \min_{j \in [m]}\max_{i \in [n]}\norm{\xv_i - \yv_j}}.
$$
\end{theorem}

Although they encode subtle features underlying the data, an important property of persistence diagrams is that they are invariant to $\o(d)$ transformations, i.e. for all $\X \in \Rnd$ and $\O \in \o(d)$, 
\begin{align}
    \dgm(\X\O) = \dgm(\X).\label{eq:pd-invariance}
\end{align}
This makes persistence diagrams particularly useful for analyzing the latent structure underlying the spectral embedding of a graph. This line of investigation was first initiated by \cite{solanki2019persistent}. In contrast, however, persistence diagrams are not invariant to some other group actions, e.g. scale transformations. 

\begin{figure}
    \centering
    \includegraphics[width=1.0\textwidth]{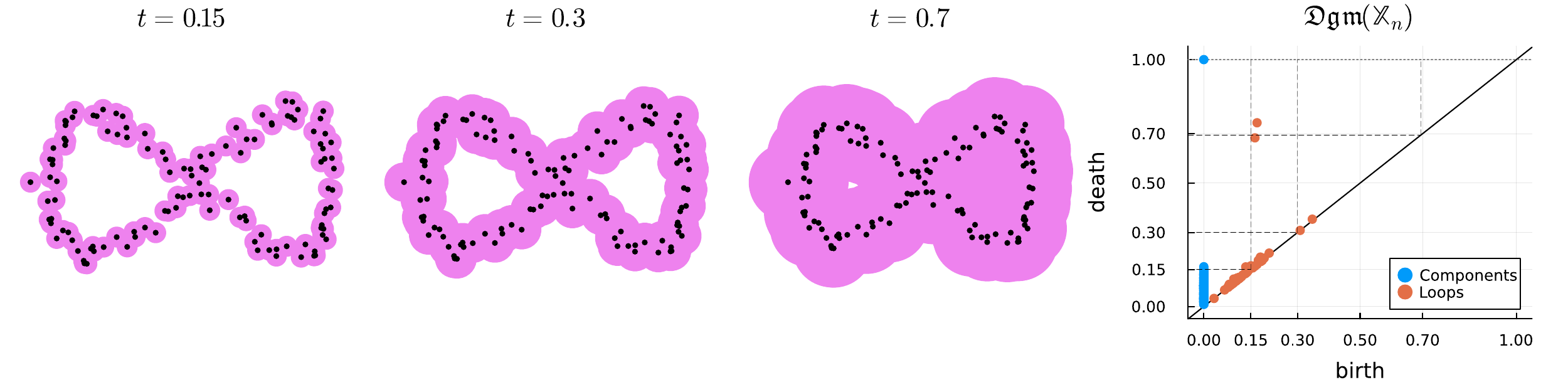}
    \caption{Illustration of a $\dgm(\X)$ computed on a point-cloud in $\X \in \Rnd$ using the \v{C}ech filtration.}
    \label{fig:cech}
\end{figure}

\subsection{Additional details on Topology-aware clustering}
\label{sec:clustering-appendix}

\citet{chazal2013persistence} proposed ToMATo clustering as an algorithm to overcome the drawbacks of $k-$means clustering by forming clusters based on topological persistence. However, their methodology \citep[Algorithm 1]{chazal2013persistence} requires the following tuning parameters: (1) A Kernel $k$ for density filtering, (2) a bandwidth $\sigma$, (3) the cutoff level $\tau$ of the density filtration, and (4) the resolution $\delta$ for the Rips complex. While their method is useful for robust clustering in the presence of outliers; we note that under differential privacy, no extraneous outliers are introduced. Therefore, we propose Algorithm~2 as a simple clustering algorithm inspired by ToMATo, which requires just one tuning parameter
$q$ --- a cutoff quantile.

Taking $q = 10$, we may filter the relevant points in the  persistence diagram $\dgm(\Xnhat)$. The filtered points in the
$0^{th}$ order persistence diagram are as shown in Figure~8\,(a). Similarly, the filtered points in the $1^{st}$ order persistence diagram are shown in Figure~8\,(b).  

Figure~9\,(a) shows the representative cycles in the $1^{st}$ order persistence diagram obtained using these filtered points from Figure~9\,(b), and coresponds to the two cycles present in the Two Circles latent positions in \cref{fig:clustering}. We can use filtered points in the persistence diagram to cluster the points using Algorithm~2, as shown in Figure~9\,(b). 

\begin{algorithm}
  \caption{Topology-aware Clustering}
  \label{alg:clustering}
  \begin{algorithmic}[1]
    \State \textbf{Input:} points $\Xn$ and a cutoff quantile $q$
    \State \textbf{Output:} Clusters $\C$
    \smallskip
    \State Compute the persistence diagram $\D = \dgm(\Xn)$
    \For {$i \in \abs{\D}$}
    \State Compute the total persistence $\delta_i = d_i - b_i$ for the birth/death pair $(b_i, d_i) \in \D$
    \State Compute the \textit{local outlier factors} $\mathcal{L} = \pb{L(\delta_i) : i \in \abs{\D}}$
    \smallskip
    \State Initialize clusters $\C \leftarrow \varnothing$
    \For {$i \in \abs{\D}$}
    {$L(\delta_i) > \text{median}(\mathcal L) + q \cdot \text{mad}(\mathcal L)$}
    \State Filter the simplex $\sigma$ with $(b_i, d_i) = \qty(\textup{birth}(\sigma), \textup{death}(\sigma))$
    \State Extract the vertices $V(\sigma) = \pb{ \xv \in \Xn : \xv \subset \sigma }$
    \State $\C \leftarrow \C \bigcup V(\sigma)$
    \smallskip
    \State \textbf{return} $\C$
  \end{algorithmic}
\end{algorithm}

\begin{figure}
  \centering
  \begin{subfigure}[b]{0.4\textwidth}
      \centering
      \includegraphics[width=\textwidth]{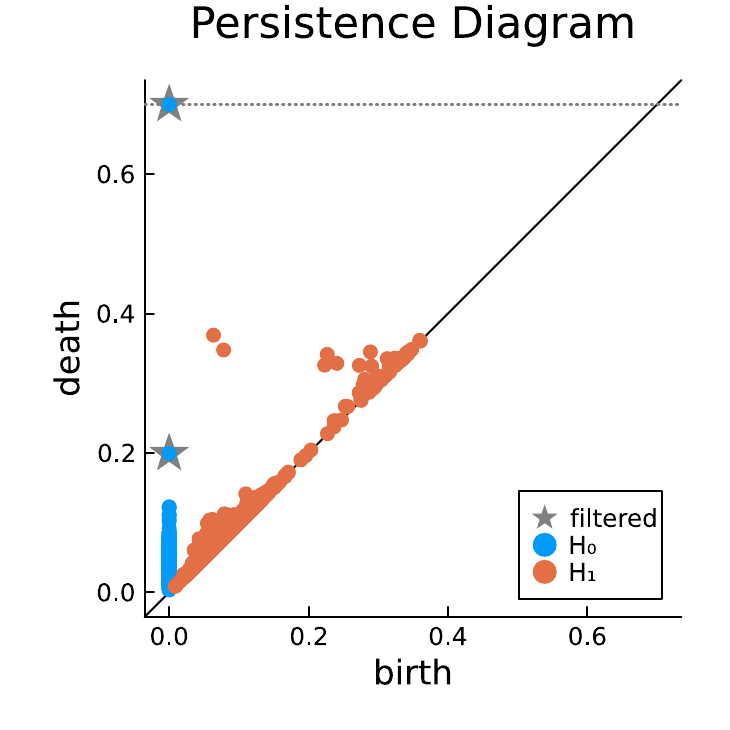}
      \caption{Filtering $H_0$ with $q=10$}
  \end{subfigure}
  \begin{subfigure}[b]{0.4\textwidth}
      \centering
      \includegraphics[width=\textwidth]{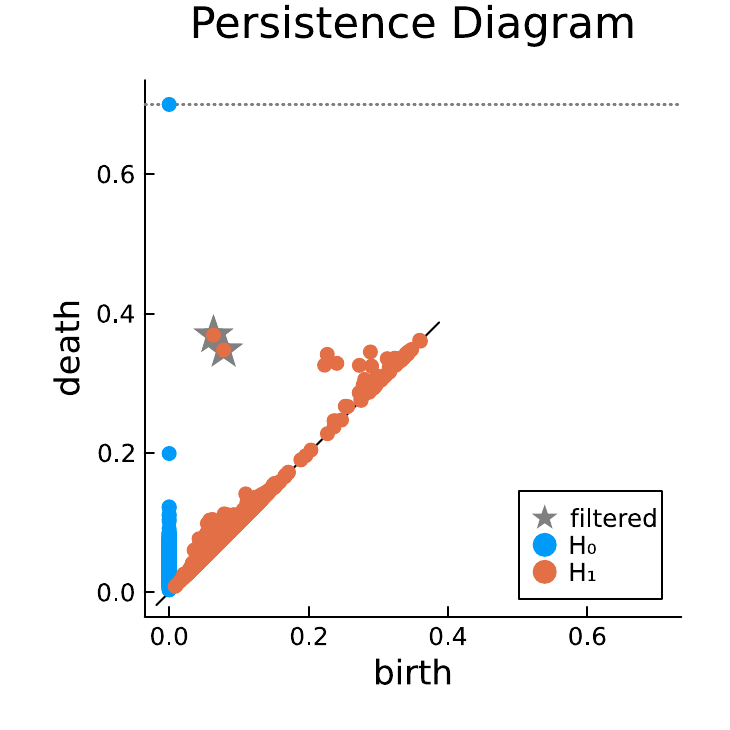}
      \caption{Filtering $H_1$ with $q=10$}
  \end{subfigure}
  \caption{Filtered points in the persistence diagrams for topology-aware clustering for the Two Circles data using Algorithm~2.}
  \label{fig:filtering}
\end{figure}

\begin{figure}
  \centering
  \begin{subfigure}[b]{0.40\textwidth}
    \centering
    \includegraphics[width=\textwidth]{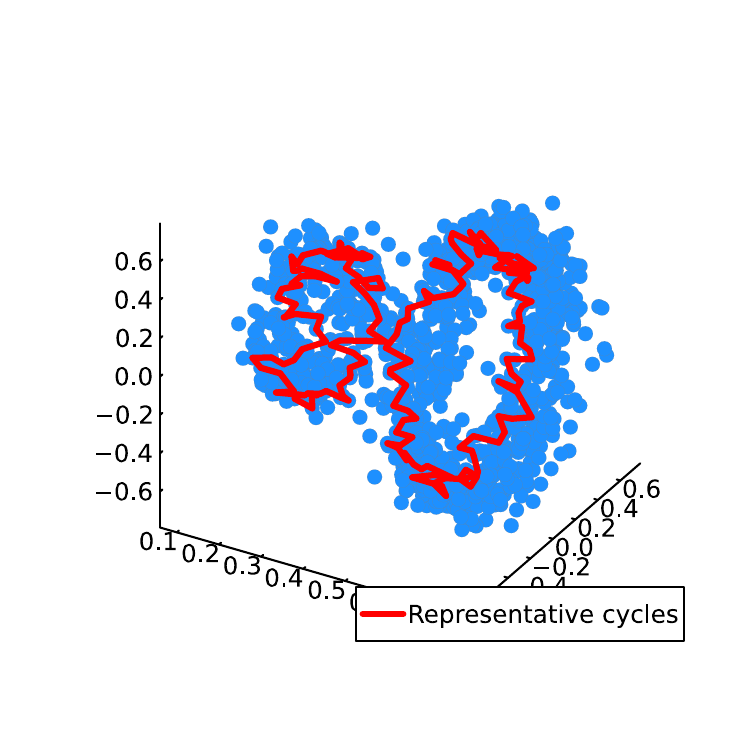}
    \caption{Representative $H_1$ cycles}
  \end{subfigure}
  \begin{subfigure}[b]{0.42\textwidth}
    \centering
    \includegraphics[width=\textwidth]{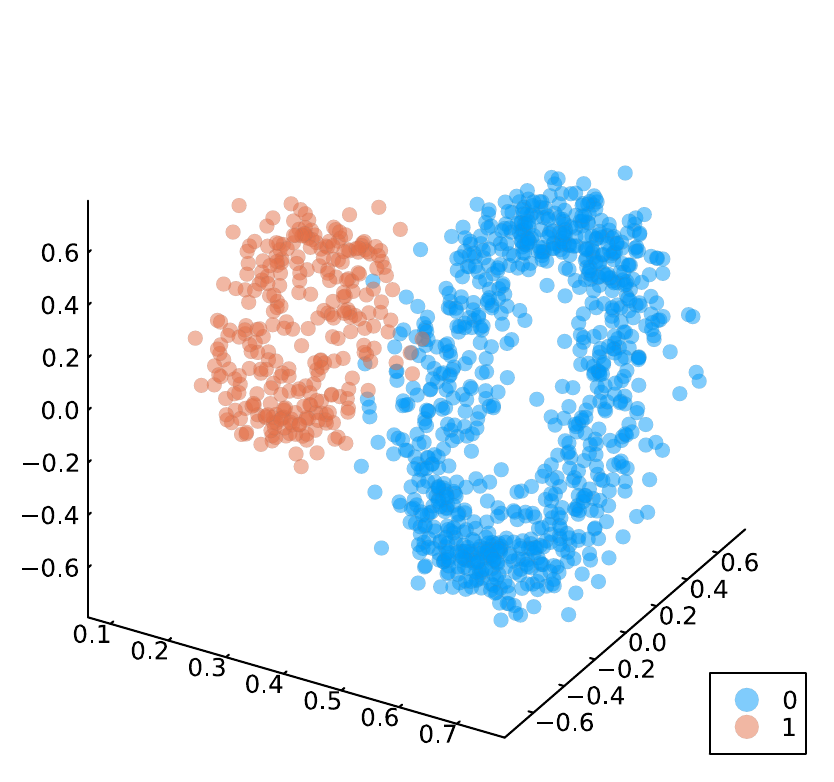}
    \caption{Topology-aware clustering using $H_0$ cycles}
  \end{subfigure}
  \caption{Representative cycles and topology-aware clustering for the Two Circles data.}
  \label{fig:representative}
\end{figure}



\end{document}